\documentclass{article}

\usepackage[final]{corl_2026} 

\usepackage[table]{xcolor}
\usepackage{float}
\usepackage{amsmath}
\usepackage{amssymb}
\usepackage{bbm}
\usepackage{graphicx}
\usepackage{booktabs}
\usepackage{makecell}
\usepackage{caption}
\usepackage{wrapfig}
\usepackage{multirow}
\usepackage{bm}
\usepackage{enumitem}
\usepackage{ragged2e}
\usepackage{pifont}
\newcounter{algorithm}
\renewcommand{\thealgorithm}{\arabic{algorithm}}
\usepackage{amsthm}

\newtheorem{theorem}{Theorem}[section]
\newtheorem{proposition}[theorem]{Proposition}
\newtheorem{lemma}[theorem]{Lemma}
\newtheorem{corollary}[theorem]{Corollary}

\title{SSR: Scaling Surefooted and Symmetric Humanoid Traversal to the Open World}

\author{
Ruiqi Yu$^{1}$\thanks{Equal contribution} \quad 
Yiwen Wang$^{1}$\footnotemark[1] \quad 
Yuan Hao$^{1}$ \quad 
Jun Wu$^{1,2}$ \quad 
Qiuguo Zhu$^{1}$\thanks{Corresponding author}
\vspace{0.5mm}
\\
$^1$Zhejiang University\\
$^2$Zhejiang Key Laboratory of Additive Manufacturing Technology and Equipment\\
}

\begin{document}
\maketitle

\vspace{-10pt}
\begin{figure}[!htbp]
  \centering
  \includegraphics[width=0.99\linewidth]{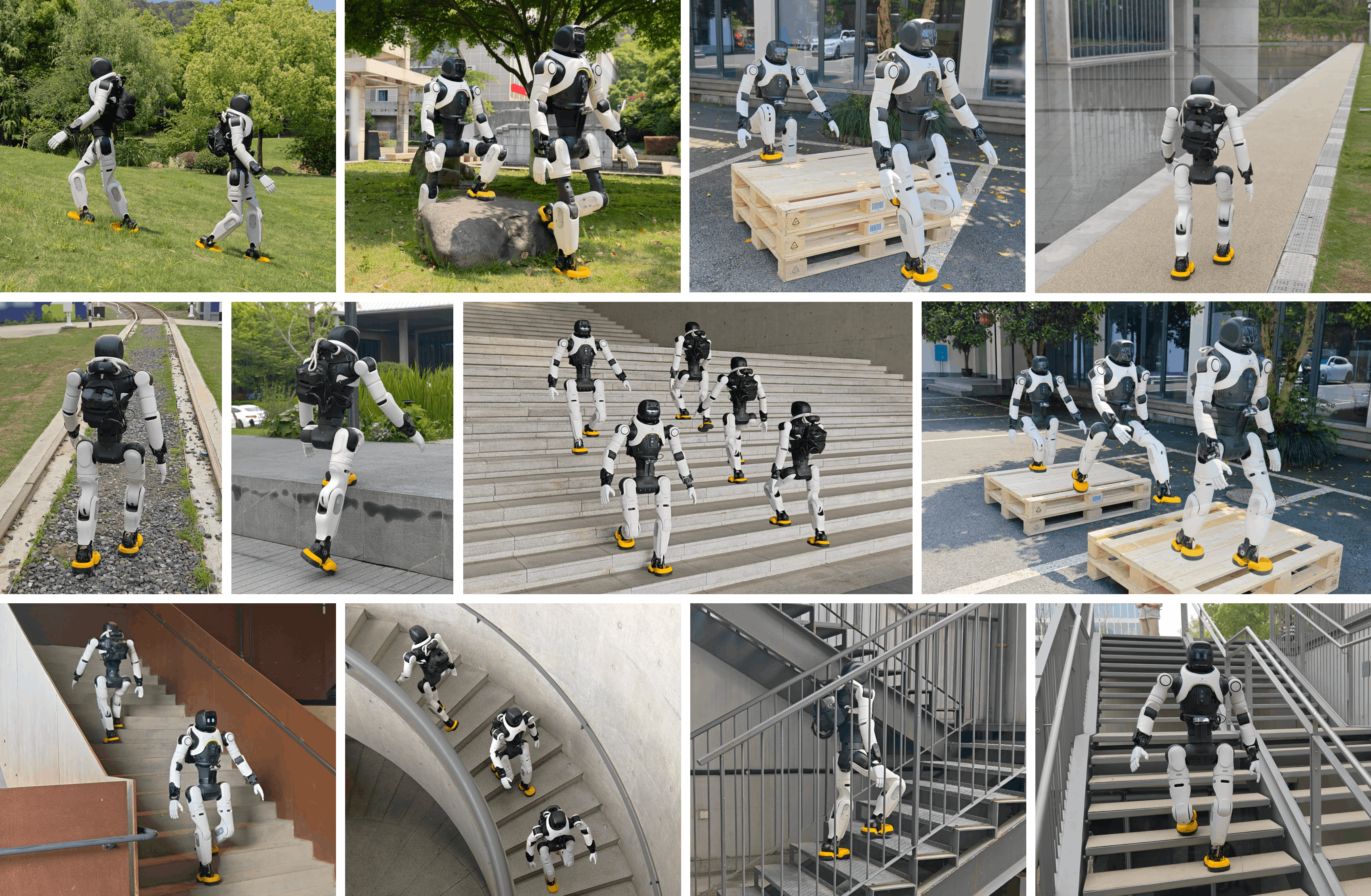}
  \vspace{-3pt}
  \captionsetup{font=footnotesize}
  \caption{SSR enables surefooted and symmetric humanoid traversal across diverse human environments. Driven by egocentric vision, the robot reliably traverses grassy slopes, stairs of varying sizes and structures, steps onto high platforms, and crosses wide gaps. Throughout long-horizon traversal, it maintains safe foot placement together with coordinated and natural whole-body motion.}
  \label{fig1}
  \vspace{-10pt}
\end{figure}


\begin{abstract}
Advancing humanoid traversal toward the open world is key to practical deployment in human environments, but remains challenging. The robot must use vision to ensure safe and reliable foot placement on heterogeneous terrain during highly dynamic motion, while producing coordinated, natural whole-body behaviors. We propose \textbf{SSR}, an efficient end-to-end framework for egocentric vision-based humanoid traversal that jointly learns these capabilities. SSR introduces \textit{imagined foothold guidance}, which learns to model forthcoming swing-foot contacts and evaluates their support to guide pre-touchdown swings toward stable regions, reducing edge slips. It further employs \textit{equivariant latent-space symmetry augmentation} to efficiently induce bilateral coordination under high-dimensional visual observations, and uses \textit{terrain-specific multi-discriminator motion priors} to encourage human-like behavior across scenes. Extensive experiments show that SSR achieves safe, stable, and high-quality locomotion on diverse real-world terrains, including stairs with varied structures and extreme challenges such as wide gaps and high platforms, while enabling reliable long-horizon traversal in diverse outdoor environments. 
Project website: \url{https://ssr-humanoid.github.io/}.

\end{abstract}



\section{Introduction}
\label{sec1}

Deploying humanoids in human environments requires traversing heterogeneous terrains with agile, stable, and natural whole-body motion, from everyday stairs to wide gaps and high platforms~\cite{radosavovic2024learning}. This demands that the robot perceive changing surroundings, place each step within safe support regions, and maintain coordinated, human-like behavior throughout long-horizon locomotion~\cite{bonnen2021binocular}.

Recent learning-based humanoid locomotion methods leverage onboard perception to adapt to terrain changes. Among them, policies conditioned on egocentric depth images are especially promising for dynamic real-world scenes, as they avoid explicit, noise-sensitive mapping~\cite{zhuang2024humanoid, zhu2026hiking, sun2026now, sun2025dpl}. However, turning such perception into robust traversal across diverse human environments remains an open challenge. 
The policy must infer local terrain structure from visual cues and convert this understanding into reliable foothold decisions. This is critical for flat-footed humanoids, since stepping near terrain edges sharply reduces the available support area and increases the risk of slips or falls.
Beyond safety, traversal in human environments also requires coordinated, controllable, and natural whole-body motion. Prior work often introduces symmetry priors~\cite{nie2026coordinated, mittal2024symmetry, su2024leveraging} or human-motion imitation~\cite{peng2021amp, tang2024humanmimic, zhang2024whole}, yet integrating them into egocentric-vision policies can be inefficient and brittle. High-dimensional visual observations make symmetry learning computationally expensive, while terrain-dependent dynamics can destabilize a shared human-like style prior. Overall, an efficient framework that jointly enables safe foot-placement behavior, symmetric and natural whole-body motion, and robust traversal across diverse terrains under vision is still lacking.

To address these challenges, we propose \textbf{SSR}, a unified single-stage end-to-end reinforcement learning (RL) framework for vision-based humanoid traversal in complex real-world environments. SSR learns directly from egocentric depth and proprioception to jointly acquire reliable foot placement and high-quality whole-body motion, coupling gait symmetry with human-like behavior.

For safe and accurate foot placement, we first introduce imagined foothold guidance during training. It learns to model forthcoming swing-foot contacts and evaluates their local support before touchdown, converting sparse contact-time safety assessment into dense predictive guidance. Unlike prior methods that assess foothold safety only at or after contact, leaving swing-phase actions with delayed feedback~\cite{zhuang2024humanoid, zhu2026hiking, cheng2024extreme}, this imagined interaction enables the policy to steer the foot early toward flat and well-supported regions, thereby reducing edge contacts, slips, and falls.

To efficiently acquire coordinated whole-body motion, we develop equivariant latent-space symmetry augmentation. It applies mirror transformations to compact latent representations rather than high-dimensional visual observations. Compared with input-level augmentation and symmetry regularization, as well as strict equivariant architectures~\cite{mittal2024symmetry, su2024leveraging, zhang2025wococo}, this strategy reduces the overhead of symmetry learning for vision-based policies while preserving flexibility for exploration. We further incorporate a terrain-specific multi-discriminator adversarial motion prior~\cite{vollenweider2023advanced} to capture terrain-conditioned styles and maintain human-like motion across scenes.

We evaluate SSR in diverse environments. As shown in Fig.~\ref{fig1}, the robot robustly traverses slopes, rough ground, and stairs with varied structures. It can climb onto platforms up to 45 cm high, about $1.6\times$ its shank length, and cross gaps as wide as 90 cm. More importantly, these capabilities transfer to complex outdoor scenes and support reliable long-horizon locomotion with precise foot placement, strong bilateral coordination, robust controllability, and natural motion.

In summary, our main contributions are as follows:

\begin{enumerate}[leftmargin=*, nosep]
\item We present a novel single-stage training framework that jointly learns safe, symmetric, and human-like motion for vision-based humanoid traversal in diverse environments.
\item We introduce imagined foothold guidance, which learns foot-contact foresight, providing dense swing-phase signals for timely correction of unsafe landing decisions and robust foot placement.
\item We develop equivariant latent-space augmentation for efficient symmetry learning, and combine it with terrain-specific style priors to jointly improve whole-body motion quality.
\item We validate the learned policy through extensive indoor and outdoor experiments, demonstrating agile and reliable humanoid locomotion over heterogeneous and challenging terrains in human environments.
\end{enumerate}
	


\section{Related Work}
\label{sec2}

\paragraph{Learning Perceptive Humanoid Locomotion.}
Data-driven RL has advanced humanoid locomotion under complex contact dynamics and environmental variation~\cite{radosavovic2024learning, wang2025beamdojo, wu2026perceptive}, and exteroceptive sensing lets controllers adapt to terrain before touchdown. Earlier methods construct height maps from LiDAR and odometry~\cite{long2025learning, cui2026pilot, ma2026cmoe, he2025attention}, but depend on accurate localization and are sensitive to latency, drift, limited update rates, and occlusion. Recent approaches infer terrain directly from egocentric depth images~\cite{zhuang2024humanoid, zhu2026hiking, sun2026now, wang2025more, zhuang2026deep}, reducing reliance on mapping. However, many vision-based humanoid systems still target structured settings or require multi-stage pipelines, expert distillation, or dedicated depth modeling and augmentation~\cite{sun2026now, wu2026perceptive, zhang2025distillation, liu2026faststair, rudin2025parkour}. In contrast, we learn a unified traversal policy from egocentric vision in a single-stage framework for diverse unstructured terrains.

\paragraph{Safe Foot Placement.}
Classical footstep control methods decompose locomotion into perception, planning, and tracking~\cite{mastalli2020motion, agrawal2022vision, fahmi2022vital}. While effective in specific settings, they rely on hand-crafted design and conservative assumptions. RL-based methods learn foot placement end to end by tracking planner-generated footholds~\cite{tsounis2020deepgait, jenelten2024dtc, yu2021visuallocomotion, gangapurwala2022rloc} or penalizing unsafe landing regions~\cite{zhuang2024humanoid, zhu2026hiking, cheng2024extreme, zhuang2023robot}. Yet these signals are often sparse, appearing only at or after contact, weakening swing-phase guidance and training efficiency~\cite{wang2025beamdojo, sutton1984temporal}. \citet{liu2026faststair} densify learning by swing-trajectory tracking, but still depend on rule-based foothold generation and focus on stairs. \citet{zhu2024learning} reward derived foothold actions to improve landing quality. We learn to imagine future footholds during swing, assessing support before touchdown to guide flatter, safer landings without planners or edge detection.

\paragraph{Symmetry Learning in Legged Locomotion.}
High-quality humanoid traversal requires coordinated bilateral motion. Prior work exploits morphological symmetry via data augmentation or regularization~\cite{mittal2024symmetry, su2024leveraging, zhang2025wococo, zhang2024learning}, but under visual inputs with short-term memory, these methods require mirrored forward passes on high-dimensional inputs or mirrored hidden-state rollout, increasing computational and memory cost. Strictly equivariant architectures reduce this overhead~\cite{nie2026coordinated}, but may limit symmetry breaking near neutral states for gait exploration \cite{abdolhosseini2019learning}. We use an equivariant encoder to move augmentation into latent space, balancing efficiency and flexibility in symmetry learning.



\section{Method}
\label{sec3}

Our goal is safe and high-quality humanoid traversal over diverse unstructured terrains directly from proprioception and raw depth in a unified, efficient framework. 
SSR achieves this with imagined foothold guidance for dense pre-contact swing correction (Sec. \ref{sec3.2}), equivariant latent-space symmetry augmentation for efficient bilateral gait learning via compact latent mirroring (Sec. \ref{sec3.3}), and cross-terrain motion priors for human-like whole-body behavior (Sec. \ref{sec3.4}). Fig. \ref{fig2} summarizes SSR.

\subsection{Single-Stage Learning of Generalizable Traversal Skills}
\label{sec3.1}

We formulate humanoid locomotion as a partially observable Markov decision process (POMDP) and optimize the policy with PPO~\cite{schulman2017proximal} under an asymmetric actor-critic architecture~\cite{nahrendra2023dreamwaq}.

\paragraph{Observations and Action.}

At timestep $t$, the policy receives proprioceptive and visual inputs. The proprioception $\mathbf o_t^p \in \mathbb{R}^{72}$ includes base angular velocity $\boldsymbol \omega_t$, projected gravity $\mathbf g_t$, velocity commands $\mathbf u_t$, joint positions $\boldsymbol \theta_t$, joint velocities $\dot{\boldsymbol\theta}_t$, and previous action $\mathbf a_{t-1}$. We stack a short history $\mathbf o_{t-h+1:t}^p$ to reduce partial observability. Together with the depth image $\mathbf I_t \in \mathbb{R}^{36 \times 36}$, this forms the policy observation $\mathbf o_t=(\smash[b]{\mathbf o_{t-h+1:t}^p},\mathbf I_t)$, which maps to actuated joint position targets $\mathbf a_t \in \mathbb{R}^{21}$. The critic additionally uses privileged information, including ground-truth base velocity $\mathbf v_t$, foot velocities $\smash[b]{\mathbf v_t^f}$, contact states $\mathbf c_t$, limb positions $\smash[b]{\mathbf p_t^h}$ and $\smash[b]{\mathbf p_t^f}$, and height maps around the feet and body, $\smash[b]{\mathbf H_t^f}$ and $\smash[b]{\mathbf H_t^b}$.

\begin{figure}[!tbp]
  \centering
  \includegraphics[width=0.88\linewidth]{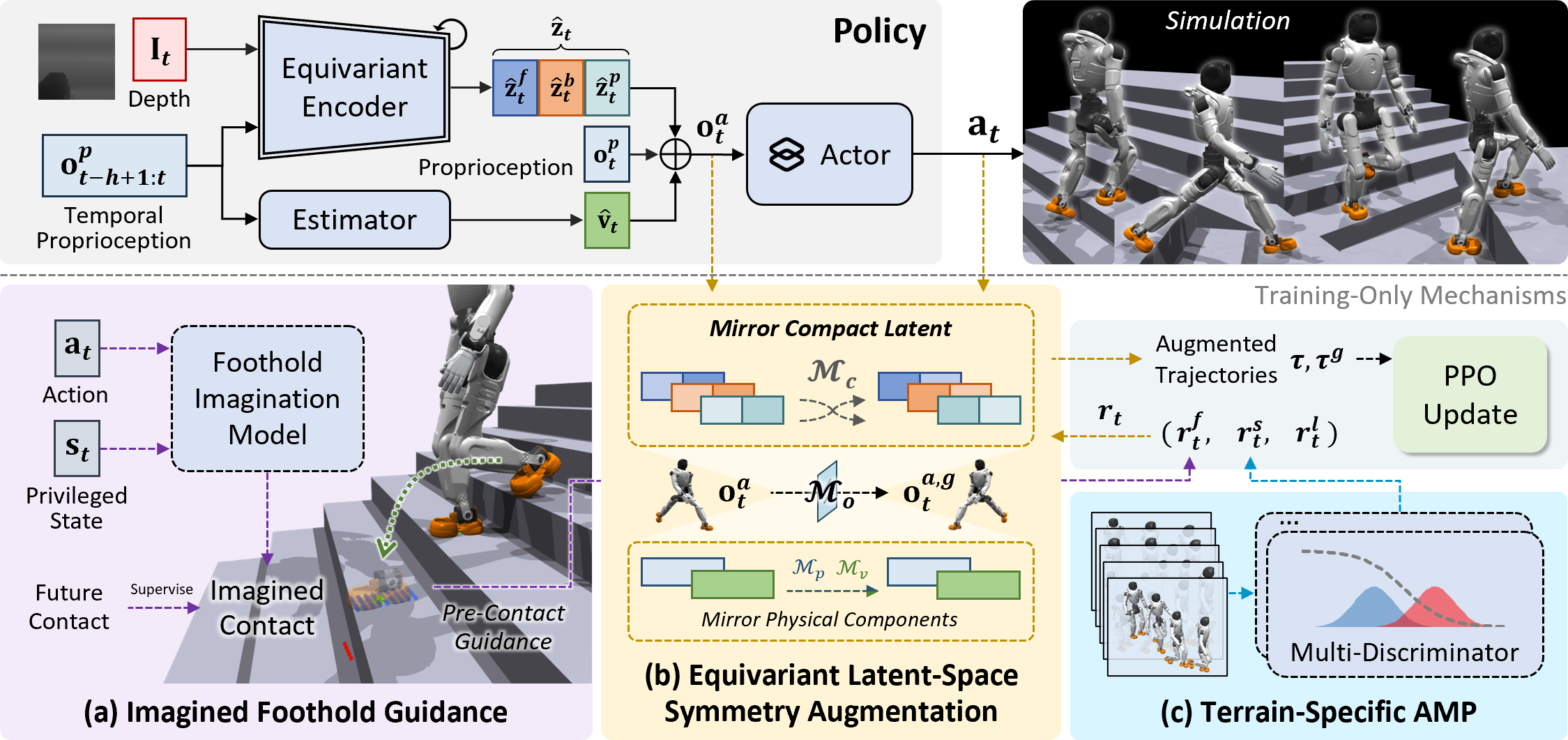}
  \vspace{-2pt}
  \captionsetup{font=footnotesize}
  \caption{\textbf{Overview of the SSR framework}. The policy combines a recurrent equivariant encoder, an estimator, and a MoE actor to learn unified traversal across terrains from egocentric depth images and temporal proprioception. During training, SSR learns surefooted, symmetric, and human-like motion through three mechanisms: 
  (a) imagined foothold guidance for dense pre-contact swing correction,
  (b) equivariant latent augmentation for efficient symmetry learning, and (c) terrain-specific multi-discriminator AMP.}
  \label{fig2}
  \vspace{-12pt}
\end{figure}

\paragraph{Policy Architecture.}

Our policy comprises a recurrent cross-modal encoder, a motion-state estimator, and a MoE actor~\cite{huang2025moe, jacobs1991adaptive}. The encoder extracts compact robot-environment context for environmental awareness and adaptability~\cite{luo2024pie,li2025kivi}. It encodes the image with a convolutional neural network (CNN) and temporal proprioception with a multi-layer perceptron (MLP), then fuses them with a gated recurrent unit (GRU) into a latent representation with three heads, $\mathbf{\hat z}_t=[\mathbf{\hat z}_t^f,\mathbf{\hat z}_t^b,\mathbf{\hat z}_t^p]^\top$. Here, $\mathbf{\hat z}_t^f$ and $\mathbf{\hat z}_t^b$ capture foot‑ and base‑centric terrain geometry and decode $\mathbf{\hat H}_t^f$ and $\mathbf{\hat H}_t^b$~\cite{yu2025start}, while $\mathbf{\hat z}_t^p$, learned with a variational autoencoder (VAE), predicts next proprioception $\mathbf{\hat o}_{t+1}^p$ to model system dynamics~\cite{nahrendra2023dreamwaq}. A separate MLP estimator predicts base velocity $\mathbf{\hat v}_t$ from temporal proprioception. During training, they are optimized with a hybrid prediction loss. We further impose mirror equivariance on the encoder to support symmetry learning in Sec. \ref{sec3.3}. The actor takes current proprioception, $\mathbf{\hat v}_t$, and $\mathbf{\hat z}_t$ as input, allowing one policy to represent terrain-dependent motion modes.

\paragraph{Reward Functions.}

We group the objective into $r^l$ for task-oriented objectives, $r^f$ for foothold guidance, and $r^s$ for human-like motion style. We use three critics to improve estimation of corresponding value functions~\cite{zargarbashi2025robotkeyframing, huang2025standingup}. See detailed reward definitions in Appendix~\ref{appendixA.3}.

\subsection{Imagined Foothold Guidance for Foresighted Foot Placement Correction}
\label{sec3.2}

\paragraph{Imagining Future Footholds During Swing.}

During training, a foothold imagination model learns to map privileged state $\mathbf{s}_t$ and action $\mathbf{a}_t$ to a prospective contact distribution $\hat{\mathbf{F}}_t = \{ \hat F_{i,t}\}_{i=1}^2$. For foot $i$, $\hat F_{i,t} = (\boldsymbol {\hat \mu}_{i,t},\hat \sigma_{i,t})$ parameterizes a Gaussian imagined-contact distribution in the current base frame: $q_{i,t}(\mathbf{p})=\mathcal{N}(\mathbf{p};\boldsymbol {\hat \mu}_{i,t},(\hat \sigma_{i,t})^2 \mathbf{I})$, where $\boldsymbol {\hat \mu}_{i,t}\in\mathbb{R}^2$ denotes the anticipated contact location and $\hat \sigma_{i,t}\in\mathbb{R}$ captures uncertainty. This probabilistic form yields smoother support estimates over the imagined landing region. The supervision target $\mathbf{p}_{i,t}^{*}$ is foot $i$'s first valid future contact location, expressed in base frame. We train the model with the Gaussian negative log-likelihood:
\begin{equation}
\begin{aligned}
\mathcal{L}_{\text{pred}}=
-\sum_{i=1}^2 \text{log}q_{i,t} (\mathbf{p}_{i,t}^{*})
\propto \sum_{i=1}^2(\frac{||\mathbf{p}_{i,t}^{*}-\boldsymbol {\hat \mu}_{i,t}||^2_2}{2(\hat \sigma_{i,t})^2}+2\text{log}\hat \sigma_{i,t}).
\end{aligned}
\end{equation}
Early instability reduces valid targets as frames without future contact are omitted. We therefore use global mean terrain level as a reliability proxy and permanently activate guidance upon exceeding 2.

\paragraph{Guiding Safer Swings Before Touchdown.}

\setlength{\columnsep}{7pt}
\setlength{\intextsep}{9pt}
\begin{wrapfigure}[9]{r}{0.53\textwidth}
  \centering
  \includegraphics[width=\linewidth]{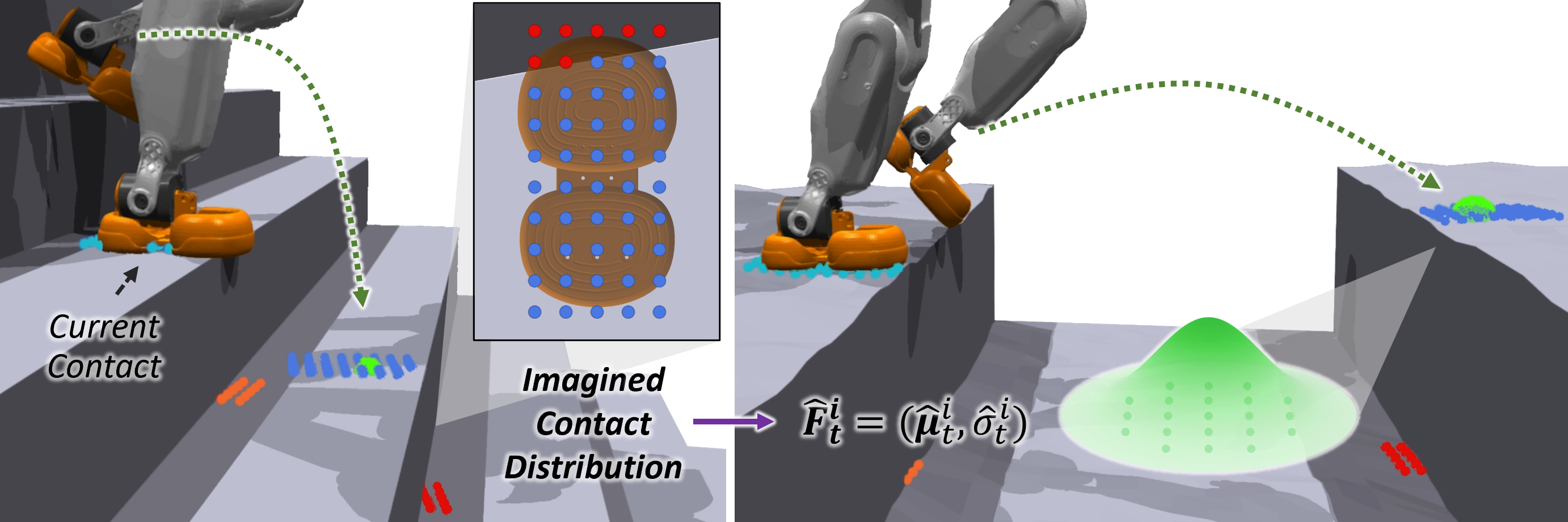}
  \vspace{-12pt}
  \captionsetup{font=footnotesize}
  \caption{\textbf{Imagined foothold guidance}. 
  We measure support deficiency over a sole-sized terrain patch. During swing, a foothold imagination model anticipates future-contact distributions for pre-contact dense guidance.}
  \label{fig3}
\end{wrapfigure}

Given the imagined landing distribution, we form the guidance signal by measuring the unsupported fraction within a sole-covered region. Let $\mathbf H^f(\mathbf{p})\!=\!\{h_k(\mathbf{p}) \}_{k=1}^n$ denote the height map of a 22.5 cm $\times $10 cm sole patch centered at $\mathbf{p}\!\in\!\mathbb{R}^2$, sampled every 2.5 cm. Here, $h_k(\mathbf p)$ is the terrain height at $\mathbf{p}+\delta_k$ and $\delta_k$ is the $k$-th offset. We define \textit{support deficiency} at foothold $\mathbf{p}$ as:
\begingroup
\begin{equation}
\rho (\mathbf{p})=1-\frac{1}{n}\sum^n_{k=1}\mathbbm{1}\{[h^f(\mathbf{p})-h_k(\mathbf{p})]<\epsilon_h\},
\end{equation}
\endgroup
where $h^f(\mathbf{p})$ is the sole height. For imagined pre-contact locations, we use $h^f(\mathbf{p})\!=\!\max_k h_k(\mathbf{p})$. Larger $\rho(\mathbf{p})$ indicates less support overlap, often near terrain edges or suspended regions. As each foot alternates between stance and swing, we instantiate this \textit{foothold guidance} in the RL objective:
\begin{equation}
\begin{aligned}
r^f_t=\text{exp}({-{(\sum_{i=1}^2 \tilde \rho_{i,t}})^2/{{\sigma_f}}}),
\quad
\tilde \rho_{i,t}=c_{i,t} \rho(\mathbf{p}_{i,t}) + (1-c_{i,t})\mathbb{E}_{\mathbf{p} \sim q_{i,t}}\rho (\mathbf{p}),
\end{aligned}
\end{equation}
where $c_{i,t}$ indicates whether foot $i$ is in stance. Stance feet are evaluated at current contacts, while swing feet use expected deficiency under the imagined-contact distribution $q_{i,t}$, approximated with discrete Gaussian-weighted samples. This formulation converts touchdown-only safety signals into dense pre-contact guidance, encouraging earlier foot-placement correction before touchdown.

\subsection{Efficient Symmetry Learning with Equivariant Latent-Space Augmentation}
\label{sec3.3}

Input-level symmetry augmentation is costly for visual recurrent policies, as each mirrored sample entails re-encoding the depth image and rolling out a mirrored memory state. We instead encode the original observation once and augment the compact actor input. The key requirement is a mirror-equivariant encoder, so that mirroring the latent is equivalent to encoding the mirrored observation.

\paragraph{Mirror Transformations.}

Let $\mathcal{M}_o$, $\mathcal{M}_s$, and $\mathcal{M}_a$ denote mirror transformations for actor observation, privileged state, and action. Since $\mathbf{o}^a_t=[\mathbf{o}^p_t,\hat{\mathbf v}_t,\hat{\mathbf z}_t]^\top$, $\mathcal{M}_o$ combines the physical proprioceptive and velocity transforms $\mathcal{M}_p$ and $\mathcal{M}_v$ with the latent transform $\mathcal{M}_c$. For latent variables without predefined mirror signs, we impose a symmetry-structured representation: each head $\smash[b]{\hat{\mathbf z}^{(k)}_t}$, $\smash[b]{k\in\{f,b,p\}}$, is organized as paired left-right channel groups, and $\mathcal{M}_c$ swaps the two groups. Detailed component-wise rules are listed in Appendix~\ref{AppendixB.1}.

\paragraph{Equivariant Latent Encoding.}

\setlength{\columnsep}{8pt}
\setlength{\intextsep}{4pt}
\begin{wrapfigure}[14]{r}{0.42\textwidth}
  \vspace{-3pt}
  \centering
  \includegraphics[width=\linewidth]{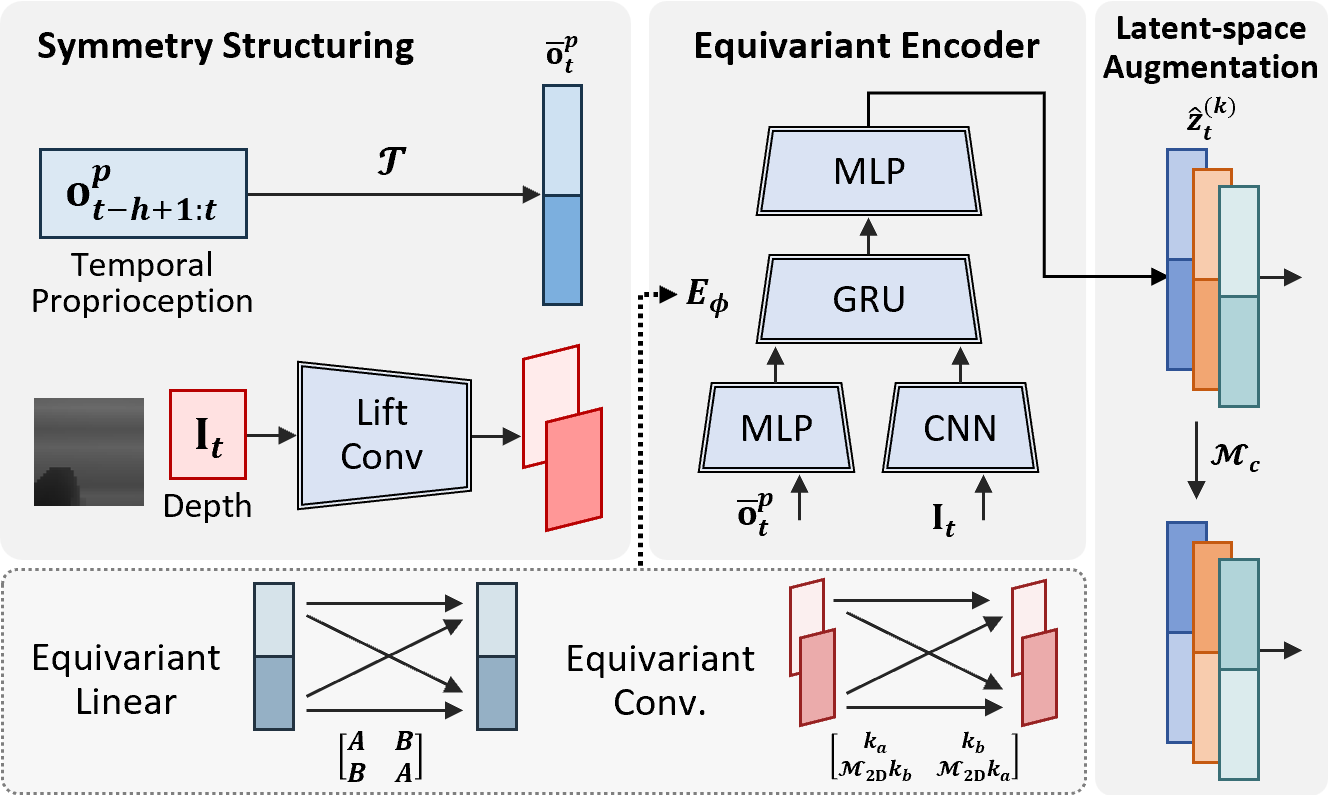}
  \vspace{-12pt}
  \captionsetup{font=footnotesize}
  \caption{\textbf{Equivariant latent-space augmentation}. We form symmetry-structured inputs and build the encoder from equivariant linear and convolutional layers, allowing each latent head to be mirrored by $\smash[b]{\mathcal{M}_c}$.}
  \label{fig4}
\end{wrapfigure}

To make this latent transform valid, the encoder must map mirrored observations to mirrored latents. We first organize its inputs into symmetry-structured form: a fixed operator $\mathcal{T}$ reorganizes temporal proprioception as $\bar{\mathbf o}^p_t=\mathcal{T}(\mathbf o^p_{t-h+1:t})$, while the first CNN layer lifts the image into paired channels. Inspired by~\citet{cesa2022a}, we then implement the MLP, CNN, and GRU with equivariant linear or convolutional layers. For each head $E_\phi^{(k)}$,
\begin{equation}
\begin{aligned}
E_\phi^{(k)}(\mathcal{M}_c\bar{\mathbf o}^p_t,\mathcal{M}_\mathrm{2D}\mathbf I_t)=\mathcal{M}_cE_\phi^{(k)}(\bar{\mathbf o}^p_t,\mathbf I_t),
\end{aligned}
\end{equation}
where $\mathcal{M}_\mathrm{2D}$ denotes horizontal reflection. Thus, the encoder compresses original inputs into a low-dimensional latent that supports channel-swap mirroring. Appendix~\ref{AppendixB.2} gives the construction and proof.

\paragraph{Latent-Space Symmetry Augmentation.}

During training, the encoder runs only on the original observation. We then mirror the compact actor observation, privileged state, and action:
$\mathbf{o}^{a,g}_t\!=\!\mathcal{M}_o\mathbf{o}^a_t$, $\mathbf{s}^g_t\!=\!\mathcal{M}_s\mathbf{s}_t$, and $\mathbf{a}^g_t\!=\!\mathcal{M}_a\mathbf{a}_t$.
This gives a mirrored trajectory $\tau^g\!=\!(\mathbf{o}^{a,g}_0,\mathbf{s}^g_0,\mathbf{a}^g_0,r_0\ldots)$, 
which is appended to the original rollout $\tau$ for PPO updates~\cite{mittal2024symmetry}. Compared with input-level augmentation, this avoids mirrored image encoding and recurrent hidden-state rollout, reducing memory and time while retaining flexible symmetry learning.

\subsection{Cross-Terrain Motion Priors with Multiple Discriminators}
\label{sec3.4}

To encourage terrain-appropriate motion style, we employ multiple discriminators~\cite{vollenweider2023advanced}, one for each terrain type $i$. Each $D_i$ models terrain-specific reference motion style from a five-frame history $\boldsymbol \psi_t$, where each frame $\mathbf s_t^{\mathrm{amp}} \in \mathbb{R}^{63}$ contains projected gravity, base linear and angular velocities, joint positions and velocities, and limb positions. Following AMP~\cite{peng2021amp}, $D_i$ is trained to distinguish reference from policy motions and produce a style reward that encourages terrain-appropriate natural behavior. We collect motion datasets for each terrain type. See details in Appendix \ref{appendixA.8}.


\section{Experiments}
\label{sec4}

\subsection{Experimental Setup}

We evaluate SSR in extensive simulation and real-world experiments to answer three questions: 
\begin{enumerate}[leftmargin=*, label=(\arabic*), nosep]
\item Can SSR learn a unified policy for high-performance traversal over diverse challenging terrains?
\item Do the key designs improve motion quality and learning efficiency?
\item Can the learned policy transfer zero-shot to unseen terrain instances and structures with long-horizon robustness and generalization?
\end{enumerate}

\paragraph{Training Environment.}

We perform single-stage training in Isaac Gym~\cite{makoviychuk2021isaac} with 4,096 AgiBot X2 humanoids on an NVIDIA RTX 4090 for about 20k iterations. We use NVIDIA Warp~\cite{macklin2022warp} to render deployment-consistent depth with terrain and self-occlusion. See Appendices~\ref{appendixA} and~\ref{appendixC}.

\paragraph{Onboard Implementation.}

On the real robot, a waist-mounted Intel RealSense D435i captures forward-facing depth at 60 Hz. Images are downsampled and cropped from $640\times360$ to $36\times36$. An onboard Jetson AGX Orin runs network inference and outputs actions at 50 Hz.

\paragraph{Compared Baselines.}
\newcommand{\method}[1]{\textit{#1}}

We compare SSR with two perceptive baselines: (1) \method{HPL}~\cite{zhuang2024humanoid}, a multi-stage egocentric vision parkour method; and (2) \method{PIM}~\cite{long2025learning}, a single-stage height-map traversal method with a hybrid internal model. We further consider ablations: (3) \method{NoFoothold} removes foothold guidance $r^f$; 
(4) \method{NoImgn} removes the imagination branch and keeps only the contact-time support assessment in $r^f$; 
(5) \method{NoSym} disables data augmentation and uses a non-equivariant encoder; (6) \method{InpSym} applies input-level mirroring with a non-equivariant encoder; as its peak memory exceeds 24 GB, this baseline alone was trained on a 48 GB RTX 4090; (7) \method{NoStyle} removes the AMP style reward $r^s$; and (8) \method{SglDisc} replaces multiple terrain-specific discriminators with a shared one.

\paragraph{Evaluation Metric and Protocol.}

We use success rate, defined as the percentage of trials that complete a 20 s traversal without termination, as the primary metric of traversability. In simulation, each method is evaluated over 1,000 randomized trials per terrain under a 1.0 m/s forward command.

\subsection{Simulation Results}

\paragraph{Overall Traversal Performance.}

\begin{figure}[!tbp]
  \centering
  \includegraphics[width=1.0\linewidth]{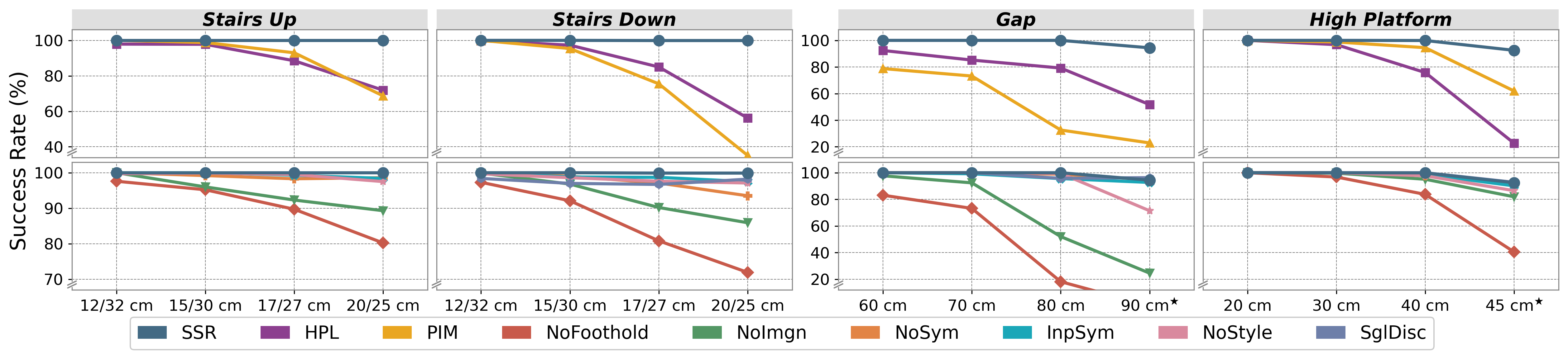}
  \vspace{-14pt}
  \captionsetup{font=footnotesize}
  \caption{Traversal performance across terrains and difficulty levels in simulation. Top row compares SSR with prior methods, and bottom row reports ablation results. \raisebox{0.2ex}{\scalebox{0.7}{$\bigstar$}} denotes difficulties beyond the training curriculum.}
  \label{fig5}
\end{figure}

Fig. \ref{fig5} shows that SSR outperforms all baselines. It maintains near-100\% success across all training difficulties, even on the hardest stairs, where a 22 cm foot has only a 3 cm support margin. It also traverses a 90 cm gap and a 45 cm platform beyond the curriculum range, demonstrating strong geometric adaptability and extrapolation. By contrast, \method{HPL} and \method{PIM} degrade rapidly as difficulty increases, suggesting that sparse or indirect foothold guidance is insufficient for precise foothold learning in these settings. Ablation results show that all SSR components matter, with the largest drops in \method{NoFoothold} and \method{NoImgn}, highlighting the importance of reliable foot placement for stable traversal.

\paragraph{Safer Foot Placement via Imagined Foothold Guidance.}

We evaluate safe foot placement with two metrics: safe foothold rate (SFR), the fraction of footholds with support ratio above 75\%, and mean support ratio (MSR), averaged over footholds. The support ratio is one minus support deficiency in Sec. \ref{sec3.2}. As shown in Fig. \ref{fig6}, SSR achieves the highest SFR on all terrains, indicating consistently safer contacts. Even on terrains requiring precise landing such as stairs down and gap, it concentrates footholds within valid support regions. The checkpoint-wise MSR curves further show SSR reaches 75\% earliest, reflecting better learning efficiency. In contrast, \method{NoFoothold} produces more aggressive gaits and the lowest safety. \method{NoImgn} uses only contact-time assessment, so swing correction is often delayed until the foot nears the ground and contacts remain edge-biased. These results show that guidance from imagined future contacts improves credit assignment for swing-phase decisions, yields consistently supported landings, and accelerates safe traversal learning.

\vspace{4pt}
\begin{figure}[!hbp]
  \centering
  \includegraphics[width=1.0\linewidth]{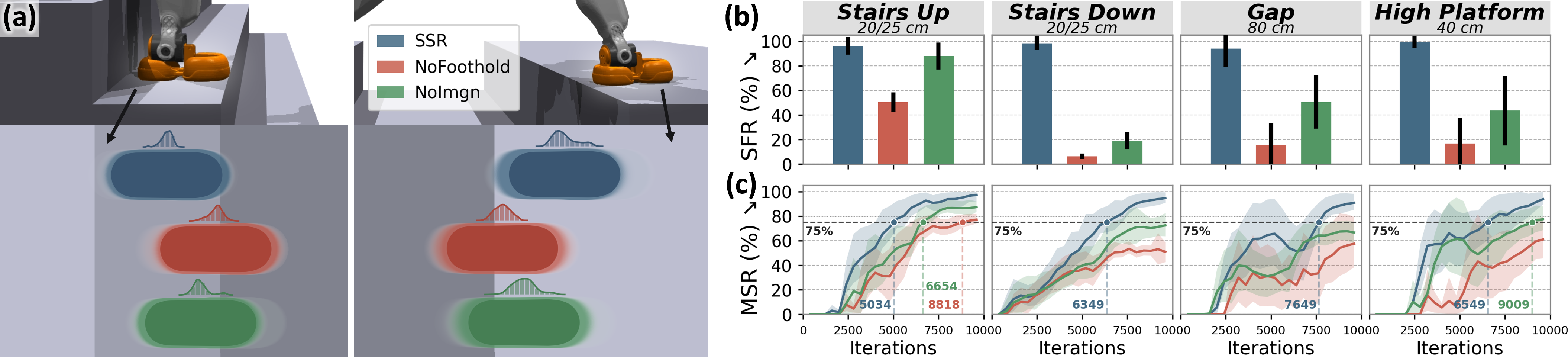}
  \vspace{-10pt}
  \captionsetup{font=footnotesize}
  \caption{Ablation study of safe foot placement learning. (a) Foothold distributions on stairs down and gap, smoothed with kernel density estimation (KDE). For each terrain, (b) the safe foothold rate (SFR); and (c) the mean support ratio (MSR) versus training iterations, evaluated using checkpoints saved every 100 iterations, with the dashed line marking the first point at which the MSR reaches 75\%.}
  \label{fig6}
\end{figure}

\paragraph{Bilateral Coordination via Efficient Latent-Space Symmetry Augmentation.}

Fig. \ref{fig7}(a) and \ref{fig7}(b) show that symmetry augmentation improves controllability and bilateral coordination. With eight-direction velocity commands, SSR tracks more accurately and, under bidirectional yaw-rate commands, stays more localized during in-place turning, with nearly mirror-symmetric trajectories. \method{NoSym} instead accumulates drift and left-right imbalance, revealing direction-dependent biases and fixed lead-leg dominance. In contrast, SSR supports both left- and right-foot-led traversal, indicating stronger bilateral skill acquisition. Fig. \ref{fig7}(c) further evaluates training efficiency. Compared with \method{InpSym}, SSR achieves higher returns within the same training time, reaches maximum difficulty earlier, and reduces GPU memory from 28.8 GB to 23.7 GB (18\%), making training feasible on one RTX 4090. This gain comes from mirroring a compact latent rather than high-dimensional inputs, avoiding extra encoding passes and hidden-state rollouts. Overall, equivariant latent-space augmentation scales symmetry-driven high-quality locomotion to visual recurrent policy learning.

\begin{figure}[!tbp]
  \centering
  \includegraphics[width=1.0\linewidth]{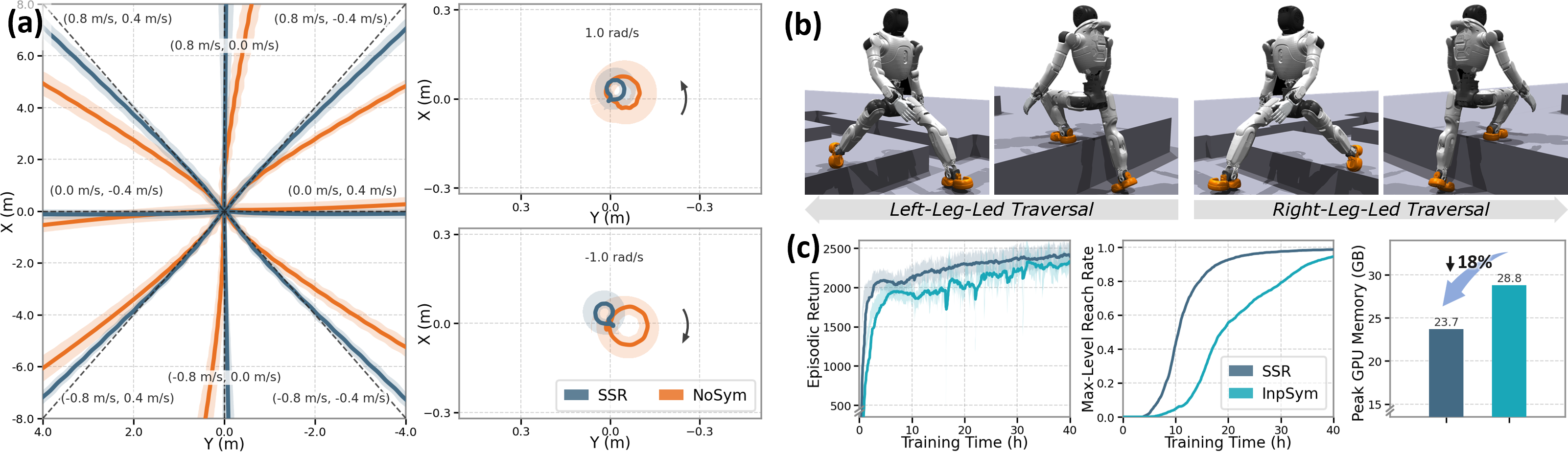}
  \vspace{-10pt}
  \captionsetup{font=footnotesize}
  \caption{Ablation study of symmetry learning. (a) Trajectories under eight linear-velocity and two in-place yaw-rate commands. (b) SSR shows either-leg-led traversal on gaps and platforms, indicating bilateral balance. (c) Training-time curves of episodic return and maximum-terrain-level reach rate, with peak GPU memory.}
  \label{fig7}
\end{figure}

\paragraph{Human-Like Motion via Terrain-Specific Style Priors.}

\newcommand{\pmstd}[1]{\raisebox{0.12ex}{\scalebox{0.78}{\mbox{$\scriptstyle\pm#1$}}}}
\newcommand{\meanstd}[2]{#1{\kern0.03em\pmstd{#2}}}
\newcommand{\meanstdbf}[2]{{\bfseries #1}{\kern0.03em\pmstd{#2}}}
\newcommand{\unit}[1]{{\fontsize{5.4}{5.4}\selectfont\,(#1)}}
\newcommand{\AP}{\textbf{AP}\unit{W}}
\newcommand{\PFCF}{\textbf{PFCF}\unit{N}}
\begin{wraptable}{r}{0.69\textwidth}
\vspace{-5pt}
\begin{minipage}{0.69\textwidth}
\centering
\renewcommand{\arraystretch}{1.2}
\setlength{\tabcolsep}{1.8pt}
\scriptsize

\begin{tabular}{
@{}l@{\hspace{2.0pt}}
c@{\hspace{1.5pt}}c@{\hspace{2.4pt}}
c@{\hspace{1.5pt}}c@{\hspace{2.4pt}}
c@{\hspace{1.5pt}}c@{\hspace{2.4pt}}
c@{\hspace{1.5pt}}c@{\hspace{1.2pt}}
}
\specialrule{1.2pt}{0pt}{1pt}
\multirow[c]{2}{*}[-0.78ex]{\textbf{Method}}
  & \multicolumn{2}{c}{\textbf{Stairs Up}}
  & \multicolumn{2}{c}{\textbf{Stairs Down}}
  & \multicolumn{2}{c}{\textbf{Gap}}
  & \multicolumn{2}{c}{\textbf{High Platform}} \\[-0.5pt]
\cmidrule(l{1.8pt}r{1.8pt}){2-3}
\cmidrule(l{1.8pt}r{1.8pt}){4-5}
\cmidrule(l{1.8pt}r{1.8pt}){6-7}
\cmidrule(l{1.8pt}r{1.8pt}){8-9}
\noalign{\vskip-1.0pt}
  & \AP & \PFCF
  & \AP & \PFCF
  & \AP & \PFCF
  & \AP & \PFCF \\
\specialrule{0.6pt}{1pt}{1pt}
SSR
  & \meanstdbf{381.3}{5.4} & \meanstdbf{507.7}{12.6}
  & \meanstdbf{342.4}{8.6} & \meanstdbf{519.2}{19.3}
  & \meanstdbf{214.4}{5.7} & \meanstdbf{463.4}{15.3}
  & \meanstdbf{253.0}{6.3} & \meanstdbf{480.4}{14.9} \\
NoStyle
  & \meanstd{451.6}{10.1} & \meanstd{538.5}{13.5}
  & \meanstd{371.4}{43.5} & \meanstd{555.5}{21.5}
  & \meanstd{317.0}{6.5} & \meanstd{518.1}{16.6}
  & \meanstd{292.3}{6.3} & \meanstd{508.7}{16.4} \\
SglDisc
  & \meanstd{391.8}{5.9} & \meanstd{544.7}{9.7}
  & \meanstd{345.3}{24.3} & \meanstd{549.4}{23.4}
  & \meanstd{241.1}{7.4} & \meanstd{481.8}{14.7}
  & \meanstd{258.2}{5.9} & \meanstd{505.8}{16.4} \\
\specialrule{1.2pt}{1pt}{0pt}
\end{tabular}
\captionsetup{
    width=\linewidth,
    singlelinecheck=false,
    skip=4pt,
    font=footnotesize
}
\caption{Ablation study of motion quality using average power (AP) and peak foot contact force (PFCF) as proxies. Results are reported as mean$_{\pm\mathrm{std}}$.}
\label{table1}
\end{minipage}
\vspace{-8pt}
\end{wraptable}

Human-like legged locomotion is associated with low mechanical cost and smooth contacts~\cite{fu2021minimizing, whittington2009simple}. We therefore report average power and peak foot contact force as motion-quality proxies. Table~\ref{table1} shows that SSR achieves the best overall metrics, indicating lower mechanical cost and softer impacts. Fig.~\ref{fig8} qualitatively illustrates coordinated, natural whole-body motion, including arm swing, across terrains, suggesting better alignment between terrain-specific motion priors and terrain dynamics.

\begin{figure}[bp!]
\centering
\newcommand{\tabstar}{\makebox[0pt][l]{\hspace{0.12em}\raisebox{0.85ex}{\scalebox{0.5}{$\bigstar$}}}}
\newsavebox{\labtablebox}
\newlength{\labtableheight}
\sbox{\labtablebox}{%
  \renewcommand{\arraystretch}{1.2}%
  \setlength{\tabcolsep}{3.5pt}%
  \scriptsize
  \begin{tabular}{lcc}
    \specialrule{1.2pt}{0pt}{0pt}
    \textbf{Terrain} & \textbf{Settings} & \textbf{Success Rate} \\
    \specialrule{0.8pt}{0pt}{-0.5pt}
    \multicolumn{3}{c}{{\fontsize{6.3pt}{7.0pt}\selectfont\textit{Standard Settings}}} \\
    \specialrule{0.6pt}{-0.5pt}{0pt}
    Stairs Up    & 15 / 30 cm     & 100.0\% {\tiny(20/20)} \\
    Stairs Down  & 15 / 30 cm     & 100.0\% {\tiny(20/20)}\\
    Gap          & 80 cm          & 95.0\% {\tiny(19/20)}\\
    Platform     & 40 cm          & 100.0\% {\tiny(20/20)}\\
    \specialrule{0.6pt}{0pt}{-0.5pt}
    \multicolumn{3}{c}{{\fontsize{6.3pt}{7.0pt}\selectfont\textit{Hard Settings}}} \\
    \specialrule{0.6pt}{-0.5pt}{0pt}
    Gap-H        & 90 cm\tabstar  & 85.0\% {\tiny(17/20)}\\
    Platform-H   & 45 cm\tabstar  & 95.0\% {\tiny(19/20)}\\
    \specialrule{1.2pt}{0pt}{0pt}
  \end{tabular}
}
\setlength{\labtableheight}{\dimexpr\ht\labtablebox+\dp\labtablebox\relax}
\begin{minipage}[t]{0.64\linewidth}
\vspace{0pt}
\centering
\includegraphics[width=\linewidth,height=\labtableheight,keepaspectratio]{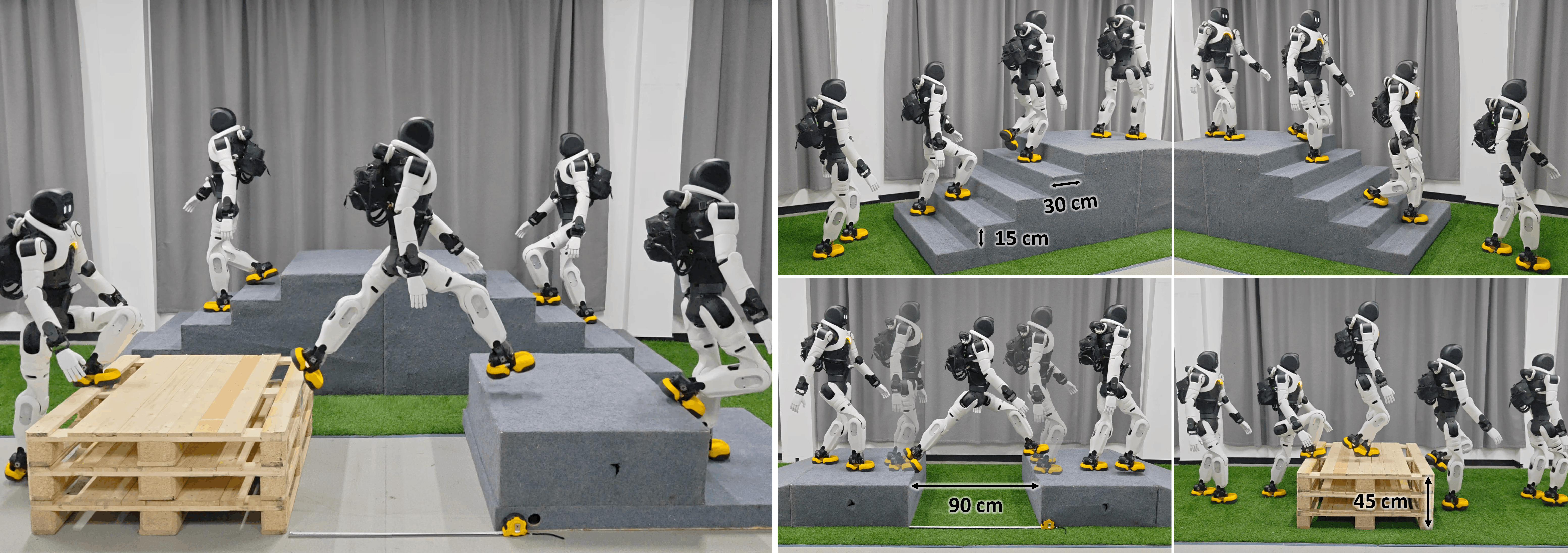}
\captionsetup{
  width=\linewidth,
  singlelinecheck=false,
  skip=4pt,
  font=footnotesize
}
\captionof{figure}{Key frames of zero-shot real-world lab-level deployment. SSR enables stair traversal, 90 cm gap crossing, and 45 cm platform climbing with safe foot placement and natural whole-body motion.}
\label{fig8}
\end{minipage}%
\hfill
\begin{minipage}[t]{0.33\linewidth}
\vspace{0pt}
\centering
\usebox{\labtablebox}
\captionsetup{
  width=\linewidth,
  singlelinecheck=true,
  skip=4pt,
  font=footnotesize
}
\captionof{table}{Real-world lab-level traversal performance. Success rates are computed over 20 trials per terrain.}
\label{table2}
\end{minipage}
\end{figure}

\subsection{Real-World Results}

\paragraph{Lab-Level Evaluation.}

We evaluate SSR's sim-to-real transfer via zero-shot deployment on the real robot. Table \ref{table2} shows high success rates on all terrains, including unseen 90 cm gap and 45 cm platform, close to simulation. To the best of our knowledge, these settings are among the most challenging real-world terrain levels reported for a single traversal policy. Fig. \ref{fig8} shows consistently precise footholds on flat support regions, coordinated gait, and natural whole-body behavior, jointly improving traversal reliability and motion quality.

\begin{figure}[!tbp]
  \centering
  \includegraphics[width=1.0\linewidth]{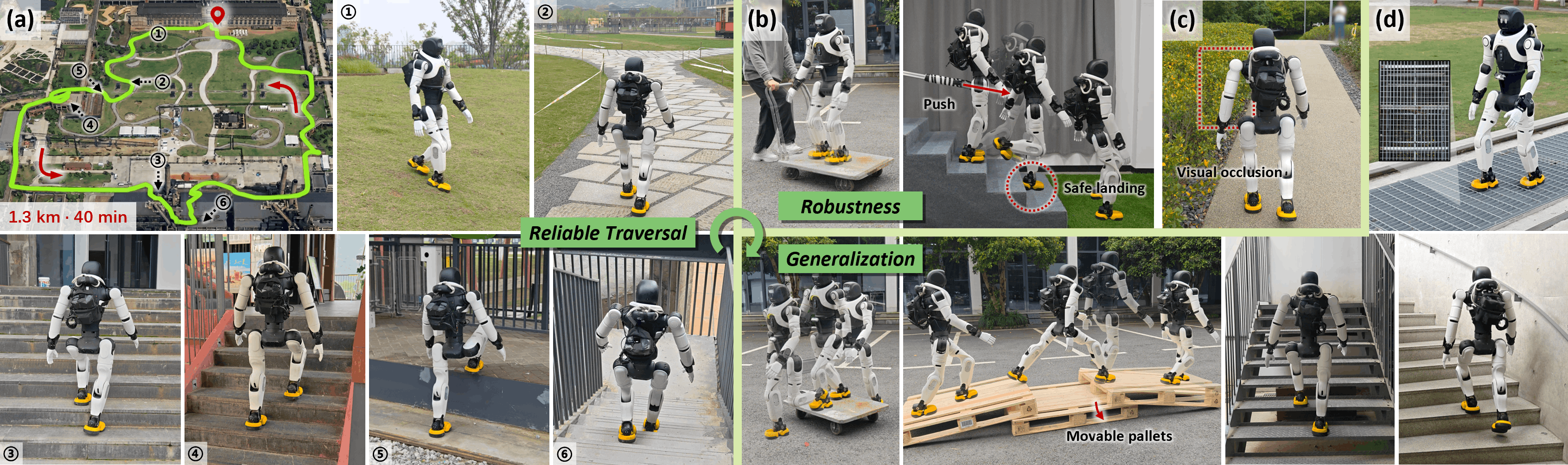}
  \captionsetup{font=footnotesize}
  \caption{(a) The humanoid completes a continuous 1.3 km, 40 min long-horizon traversal in an industrial heritage park, reliably crossing stairs of varying sizes, high platforms, rough ground, and grassy slopes. We further test robustness to (b) external disturbances, including wheeled-trolley shaking and pushes on descending stairs, and (c) visual occlusions from tall grass. (d) Zero-shot transfer to OOD terrains, including perforated grid flooring, a slippery trolley, movable forklift pallets, narrow 1.3 m-wide open-riser stairs, and spiral stairs.}
  \label{fig9}
\end{figure}

\paragraph{Outdoor Deployment.}

We validate SSR in extensive outdoor environments, as shown in Fig. \ref{fig1} and Fig. \ref{fig9}. In an outdoor industrial-park route, the robot completes a continuous 1.3 km traversal across diverse terrains without a reset or misstep, maintaining safe and coordinated locomotion over 40 min. Fig. \ref{fig9}(b)--(d) further presents three qualitative stress tests. \textbf{Perturbation Recovery}: the robot remains stable under horizontal and torsional shaking; when pushed from behind during stair descent, it still lands on the flat tread, showing timely and accurate foothold decisions under limited response time. \textbf{Visual Occlusion Adaptation}: under grass-induced occlusion, it shows no obvious gait degradation and actively steers away from poorly observed regions. \textbf{Zero-Shot Out-of-Distribution (OOD) Generalization}: despite no prior exposure to perforated grids, slippery trolleys, or movable pallets during training, the policy traverses them stably. On narrow open-riser stairs and spiral stairs, it handles visual interference from holes and handrails, avoids sidewalls that could trap the feet, and adapts to varying stair widths, demonstrating strong generalization to unseen terrain structures.

\paragraph{Cross-Platform Validation.}

To evaluate embodiment-level applicability, we instantiate SSR on a full-size DEEP Robotics DR02 humanoid. Despite major differences from AgiBot X2, Appendix~\ref{AppendixE} reports high DR02 success rates, suggesting that SSR is not tied to a single humanoid platform.


\section{Conclusion}
\label{sec5}

In this paper, we present SSR, a unified humanoid traversal framework that maps egocentric vision to stable, high-quality whole-body locomotion across heterogeneous terrain in complex human environments. 
Through imagined foothold guidance, SSR models forthcoming foot-terrain contacts to steer foresighted swing-phase foot-placement correction, enabling precise stepping in dynamic, complex settings.
Our framework also leverages equivariant latent-space symmetry augmentation to efficiently learn bilateral coordination, while terrain-specific multi-discriminator AMP further encourages consistently natural motion across terrains. Extensive experiments demonstrate that SSR generalizes surefooted, symmetric, and human-like traversal to a broad spectrum of challenging scenarios, supporting reliable long-horizon real-world deployment.


\section{Limitations and Future Work}
\label{sec6}

While SSR traverses diverse terrains effectively, several limitations remain. First, its fixed forward-facing depth camera limits perceptual coverage, while severe depth-sensing failures, such as specular reflections under strong sunlight, can still distort local geometry. Multi-view sensing and more robust perceptual embeddings could improve coverage and robustness. Second, the unimodal Gaussian in imagined foothold guidance may not represent multiple separated supports on sparse-foothold terrains such as stepping stones. Mixture models or spatial heatmaps could extend it to multimodal contacts. Third, pronounced embodiment asymmetry or inherently side-specific tasks may violate the mirror-symmetry assumption used for augmentation. Fourth, SSR mainly focuses on foot-ground interaction, leaving upper-body contacts underexplored. Extending the framework to multi-contact skills and richer human demonstrations could further expand humanoid mobility.


\clearpage
\acknowledgments{
We thank Qianshi Wang for assistance with the real-world experiments, and gratefully acknowledge AgiBot and DEEP Robotics for providing hardware support. This work was supported in part by the ``Leading Goose'' R\&D Program of Zhejiang under Grant 2023C01177, the National Key R\&D Program of China under Grant 2022YFB4701502, and the 2035 Key Technological Innovation Program of Ningbo City under Grant 2024Z300.
}


\bibliography{main}  

\newpage
\appendix


\section{Policy Training Details}
\label{appendixA}

\subsection{Terrain Curriculum}
\label{appendixA.1}

As shown in Fig.~\ref{fig10}, we generate five types of terrain for policy learning. Each terrain type contains 20 difficulty levels, with the difficulty progressively increased according to the curriculum strategy in \citep{rudin2022learning}. Each terrain instance is constructed over a rectangular region measuring $8\,\mathrm{m} \times 4\,\mathrm{m}$. 

\begin{figure}[H]
\centering
\newsavebox{\terraintablebox}
\newlength{\terraintableheight}
\sbox{\terraintablebox}{%
  \renewcommand{\arraystretch}{1.25}%
  \setlength{\tabcolsep}{5pt}%
  \footnotesize
  \begin{tabular}{ccc}
    \specialrule{1.2pt}{0pt}{1pt}
    \textbf{Terrain Type} & \textbf{Parameter} & \textbf{Range} \\
    \specialrule{0.8pt}{1pt}{0pt}
    Slope & Incline & $[0,\,20]^\circ$ \\
    \specialrule{0.4pt}{0pt}{0pt}
    \multirow[c]{2}{*}{Stairs} & Step length & $[0.25,\,0.4]\,\mathrm{m}$ \\
                               & Step height & $[0.05,\,0.2]\,\mathrm{m}$ \\
    \specialrule{0.4pt}{0pt}{0pt}
    Discrete & Obstacle height & $[0.05,\,0.2]\,\mathrm{m}$ \\
    \specialrule{0.4pt}{0pt}{0pt}
    \multirow[c]{2}{*}{Gap} & Width & $[0.1,\,0.8]\,\mathrm{m}$ \\
                             & Depth & $[0.3,\,0.6]\,\mathrm{m}$ \\
    \specialrule{0.4pt}{0pt}{0pt}
    Platform & Height & $[0.05,\,0.4]\,\mathrm{m}$ \\
    \specialrule{1.2pt}{0pt}{0pt}
  \end{tabular}%
}

\setlength{\terraintableheight}{\dimexpr\ht\terraintablebox+\dp\terraintablebox\relax}

\begin{minipage}[t]{0.51\linewidth}
\vspace{0pt}
\centering
\includegraphics[width=\linewidth,height=\terraintableheight,keepaspectratio]{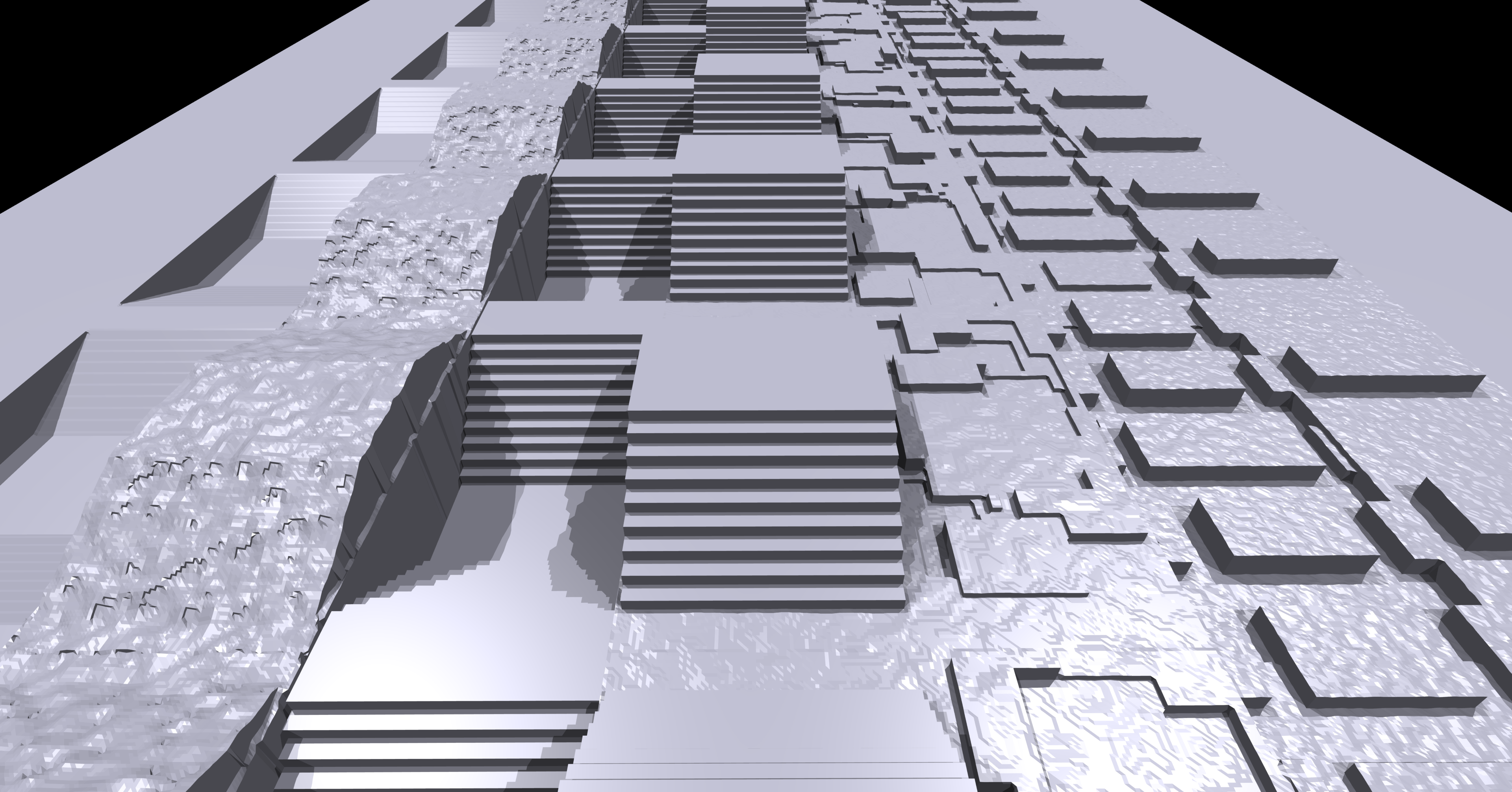}
\captionsetup{
  width=\linewidth,
  justification=centering,
  singlelinecheck=false,
  font=footnotesize,
  skip=4pt
}
\captionof{figure}{Overview of terrain types used for training.}
\label{fig10}
\end{minipage}%
\hspace{0.02\linewidth}%
\begin{minipage}[t]{0.47\linewidth}
\vspace{0pt}
\centering
\usebox{\terraintablebox}
\captionsetup{
  width=\linewidth,
  justification=centering,
  singlelinecheck=false,
  font=footnotesize,
  skip=4pt
}
\captionof{table}{Parameter ranges in the terrain curriculum.}
\label{table3}
\end{minipage}

\end{figure}

\subsection{Commands}
\label{appendixA.2}

At timestep $t$, the locomotion command is defined as 
$\mathbf{u}_t=[v_x^{\mathrm{cmd}},v_y^{\mathrm{cmd}},\omega_z^{\mathrm{cmd}}]^\top$. During training, the robot is commanded to traverse the terrain along a straight path aligned with the $x$-axis of the world frame while tracking a target heading. The yaw angular velocity command $\omega_z^{\mathrm{cmd}}$ is computed from the heading error, while the translational commands $v_x^{\mathrm{cmd}}$ and $v_y^{\mathrm{cmd}}$ are obtained by projecting the prescribed forward velocity from the world frame into the base frame. The terrain-specific ranges of the forward velocity and target heading are summarized in Table~\ref{table4}.

\begin{table}[H]
\footnotesize
\centering
\vspace{2pt}
\renewcommand{\arraystretch}{1.25}
\begin{tabular}{ccc}
\specialrule{1.2pt}{0pt}{1pt}
\textbf{Terrain Type} & \textbf{Forward Velocity (m/s)} & \textbf{Heading (°)} \\
\specialrule{0.8pt}{1pt}{0pt}
Slope & $\mathcal{U}(-1, 1)$ & $\mathcal{U}(-80, 80)$ \\
Stairs, discrete &$\mathcal{U}(0, 1.2)$ & $\mathcal{U}(-30, 30)$\\
Gap, platform &$\mathcal{U}(0, 1.2)$ & $\mathcal{U}(-15, 15)$\\
\specialrule{1.2pt}{0pt}{0pt}
\end{tabular}
\captionsetup{
    width=\linewidth,
    justification=centering,
    singlelinecheck=false,
    font=footnotesize,
    skip=4pt
}
\caption{Ranges of the forward velocity and target heading.}
\label{table4}
\end{table}

To encourage omnidirectional locomotion, we further uniformly sample $\mathbf{u}_t$ in the base frame from a predefined command space after the robot has cleared the highest terrain difficulty and exited the terrain region. This stage exposes the policy to diverse combinations of forward, lateral, and turning motions beyond structured terrain-traversal commands. The corresponding command ranges are listed in Table~\ref{table5}.

\begin{table}[H]
\centering
\renewcommand{\arraystretch}{1.25}
\footnotesize
\begin{tabular}{cc}
\specialrule{1.2pt}{0pt}{1pt}
\textbf{Command Term} & \textbf{Range} \\
\specialrule{0.8pt}{1pt}{0pt}
Forward velocity & $\mathcal{U}(-1, 1)$ m/s \\
Lateral velocity & $\mathcal{U}(-0.4, 0.4)$ m/s \\
Yaw angular velocity & $\mathcal{U}(-1.2, 1.2)$ rad/s \\
\specialrule{1.2pt}{0pt}{0pt}
\end{tabular}
\captionsetup{
    width=\linewidth,
    justification=centering,
    singlelinecheck=false,
    font=footnotesize,
    skip=4pt
}
\caption{Ranges of the locomotion commands.}
\label{table5}
\end{table}

\subsection{Rewards}
\label{appendixA.3}

Table~\ref{table6} summarizes the reward terms used for policy learning, grouped by reward category. For each term, the table provides its mathematical expression and coefficient. It also reports the advantage weight associated with each reward group in the multi-critic formulation. Table~\ref{table7} defines the main symbols used in these reward expressions.

\begin{table}[H]
\centering
\setlength{\tabcolsep}{3pt}
\renewcommand{\arraystretch}{1.25}
\footnotesize
\makebox[\textwidth][c]{
\begin{tabular}{llc}
\specialrule{1.2pt}{0pt}{1pt}
\textbf{Term} & \textbf{Expression} & \textbf{Weight} \\
\specialrule{0.8pt}{1pt}{0pt}
\rowcolor{gray!15}
\multicolumn{1}{l}{(a) \textit{Locomotion Reward}} & \multicolumn{1}{l}{$r^{l}$} & \multicolumn{1}{c}{$w_{l} = 1.0$} \\
\specialrule{0.4pt}{0pt}{0pt}
Linear Velocity Tracking & $\exp\,\bigl(-\lVert \mathbf{v}_{xy}-\mathbf{v}_{xy}^{\mathrm{cmd}}\rVert_2^{2}/\sigma_v\bigr)$ & 1.0 \\
Angular Velocity Tracking & $\exp\,\bigl(-(\boldsymbol \omega_{\mathrm{yaw}}-\boldsymbol \omega_{\mathrm{yaw}}^{\mathrm{cmd}})^{2}/\sigma_{\omega}\bigr)$ & 0.8 \\
Orientation & $\exp\,\bigl(-\lVert \mathbf{g}_{xy}\rVert^{2}/\sigma_g\bigr)$ & 0.5 \\
Angular Velocity (xy) & $\exp\,\bigl(-\lVert \boldsymbol{\omega}_{xy}\rVert_2^{2}/\sigma_{\omega}\bigr)$ & 0.25 \\
Base Height & $\exp\,\bigl(-(h_b-h_b^{\mathrm{tar}})^{2}/\sigma_h\bigr)$ & 0.4 \\
Action Rate & $\frac{1}{N}\lVert \mathbf a_t-\mathbf a_{t-1}\rVert_2^2$ & $-0.12$ \\
Smoothness & $\frac{1}{N}\lVert \mathbf a_t-2\mathbf a_{t-1}+\mathbf a_{t-2}\rVert_2^2$ & $-0.06$ \\
DoF Velocity & $\frac{1}{N}\lVert \dot{\boldsymbol{\theta}}/\dot{\boldsymbol{\theta}}^{\max}\rVert_2^{2}$ & $-0.96$ \\
DoF Torque & $\frac{1}{N}\lVert \boldsymbol{\tau}/\boldsymbol{\tau}^{\max}\rVert_2^{2}$ & $-0.6$ \\
DoF Deviation  & $\frac{1}{N}\lVert \boldsymbol{\theta}-\boldsymbol{\theta}^{\mathrm{def}}\rVert_{1}$ & $-1.8$ \\
DoF Position Limits & $\frac{1}{N}\lVert \mathbbm{1}\!(\boldsymbol{\theta}\notin[\boldsymbol{\theta}^{\min},\boldsymbol{\theta}^{\max}])\rVert_{1}$ & $-1.5$ \\
DoF Velocity Limits & $\frac{1}{N}\lVert \mathbbm{1}\!(\lvert\dot{\boldsymbol{\theta}}\rvert>\dot{\boldsymbol{\theta}}^{\max})\rVert_{1}$ & $-6.0$ \\
DoF Torque Limits & $\frac{1}{N}\lVert \mathbbm{1}\!(\lvert\boldsymbol{\tau}\rvert>\boldsymbol{\tau}^{\max})\rVert_{1}$ & $-6.0$ \\
Stand Still & $\frac{1}{N}\lVert \boldsymbol{\theta}-\boldsymbol{\theta}^{\mathrm{def}}\rVert_{1}\cdot \mathbbm{1}\!(\lVert \mathbf{u}_t\rVert_2\leq\epsilon)$ & $-0.12$ \\
Single Support & $\mathbbm{1}\![
\lVert \mathbf{u}_t\rVert_2 \le \epsilon
\;\lor\;
\exists\, t' \in [t-0.2\,\mathrm{s},\,t]\ \text{s.t.}\ n_{t'}^{\mathrm{con}} = 1
]$ & 0.2 \\
Impact Velocity & $\sum_{i=1}^2 {v^f_{z,i}}^{2} \cdot c_{i}$ & $-1.3$ \\
Contact Slippage & $\sum_{i=1}^2 \lVert\mathbf{v}^f_{xy,i}\rVert_2^{2}\cdot c_{i}$ & $-0.2$ \\
Feet Air Time & $\sum_{i=1}^2 \max(t_{\mathrm{air},i}-t_{\mathrm{air}}^{\mathrm{tar}},0)\cdot c_{i}^{\mathrm{first}}$ & $-2.0$ \\
Feet Stumble & $\sum_{i=1}^2 \mathbbm{1}(\lVert \mathbf{F}^f_{xy,i}\rVert_2 >3\lvert {F}^f_{z,i}\rvert)$ & $-2.0$ \\
Feet Lateral Distance & $\exp\,\bigl(\min(d_{y}-d_{y}^{\min},0)/\sigma_d\bigr)$ & 0.08 \\
\specialrule{0.4pt}{0pt}{0pt}

\rowcolor{gray!15}
\multicolumn{1}{l}{(b) \textit{Foothold Reward}} & \multicolumn{1}{l}{$r^{f}$} & \multicolumn{1}{c}{$w_{f} = 0.25$} \\
\specialrule{0.4pt}{0pt}{0pt}
Imagined Foothold Guidance$^1$ & $\exp\,\bigl({-{(\sum_{i=1}^2 \tilde \rho_{i,t}})^2/ \sigma_f}\bigr)$ & 1.0 \\
\specialrule{0.4pt}{0pt}{0pt}

\rowcolor{gray!15}
\multicolumn{1}{l}{(c) \textit{Style Reward}} & \multicolumn{1}{l}{$r^{s}$} & \multicolumn{1}{c}{$w_{s} = 0.2$} \\
\specialrule{0.4pt}{0pt}{0pt}
AMP & $\max[0,\;1-0.25(D_i(\boldsymbol \psi_t)-1)^2]$ & 1.0 \\
\specialrule{1.2pt}{0pt}{0pt}
\end{tabular}
}

\vspace{1pt}
\begin{minipage}{0.87\textwidth}
\scriptsize
$^1$ On continuous slopes, we set $\tilde{\rho}_{i,t}=0$ for both feet ($r_t^f=1$), since the patch-maximum sole-height estimate can misinterpret valid slope-induced height variation as missing support; the metric is intended for locally near-horizontal regions.
\end{minipage}

\captionsetup{
    width=\linewidth,
    justification=centering,
    singlelinecheck=false,
    font=footnotesize,
    skip=4pt
}
\caption{Reward terms and weights.}
\label{table6}
\end{table}

\begin{table}[H]
\centering
\renewcommand{\arraystretch}{1.25}
\footnotesize
\begin{tabular}{cl}
\specialrule{1.2pt}{0pt}{1pt}
\textbf{Symbol} & \textbf{Description} \\
\specialrule{0.8pt}{1pt}{0pt}
$\sigma_v,\,\sigma_{\omega},\,\sigma_g,\,\sigma_h,\,\sigma_d,\,\sigma_f$ & Gaussian variance scales, set to 0.25, 0.25, 0.01, 0.01, 0.03, and 0.0625. \\
$\boldsymbol{\theta}^{\min},\,\boldsymbol{\theta}^{\max}$ & Lower and upper joint position limits. \\
$\boldsymbol{\theta}^{\mathrm{def}}$ & Default joint positions. \\
$\dot{\boldsymbol{\theta}}^{\max}$ & Maximum allowable joint velocities. \\
$\boldsymbol{\tau}$ & Computed joint torques. \\
$\boldsymbol{\tau}^{\max}$ & Maximum allowable joint torques. \\
$\epsilon$ & Threshold for identifying zero-command, set to 0.15. \\
$h_b,\ h_b^{\mathrm{tar}}$ & Base height and target base height relative to the ground. $h_b^{\mathrm{tar}}=0.68$. \\
$t_{\mathrm{air},i},\ t_{\mathrm{air}}^{\mathrm{tar}}$ & Air time of foot $i$ and target air time. $t_{\mathrm{air}}^{\mathrm{tar}}=0.4$. \\
$c_{i}^{\mathrm{first}}$ & Binary indicator of whether foot $i$ makes its first contact with the ground. \\
$n_{t'}^{\mathrm{con}}$ & Number of feet in contact with the ground at time $t'$. \\
$\mathbf{F}^f_i$ & Contact force on foot $i$. \\
$\mathbf{v}^f_i$ & Velocity of foot $i$. \\
$d_{y}$ & Lateral distance between the two feet. \\
$d_{y}^{\min}$ & Minimum allowable lateral distance between two feet, set to 0.22. \\
\specialrule{1.2pt}{1pt}{0pt}
\end{tabular}
\captionsetup{
    width=\linewidth,
    justification=centering,
    singlelinecheck=false,
    font=footnotesize,
    skip=4pt
}
\caption{Symbols in the reward definitions.}
\label{table7}
\end{table}

\newpage
\subsection{Network Architectures}
\label{appendixA.4}

Table~\ref{table8} summarizes the network architectures used in our framework.

\newcommand{\tightcmidrule}[1]{
  \noalign{\vskip-\aboverulesep}
  \cmidrule[0.4pt]{#1}
  \noalign{\vskip-\belowrulesep}
}
\begin{table}[H]
\centering
\vspace{5pt}
\renewcommand{\arraystretch}{1.25}
\scriptsize
\begin{tabular}{cccc}
\specialrule{1.2pt}{0pt}{1pt}
\textbf{Module} & \textbf{Submodule} & \textbf{Component} & \textbf{Structure} \\
\specialrule{0.8pt}{1pt}{0pt}
\multirow{7}{*}{\shortstack{Asymmetric\\Actor-Critic}}
& \multirow{4}{*}{MoE Actor}
& Number of experts & 5 \\
& 
& Expert MLP hidden size & [1024, 512, 128] \\
& 
& Gate MLP hidden size & 128 \\
& 
& Activation & ELU \\
\tightcmidrule{2-4}
& \multirow{3}{*}{Multi-Critic}
& Number of critics & 3 \\
& 
& MLP size & [512, 256, 128] \\
& 
& Activation & ELU \\
\specialrule{0.4pt}{0pt}{0pt}
\multirow{14}{*}{Encoder}
& \multirow{3}{*}{Proprioception Encoder}
& MLP hidden size & [512, 256, 128] \\
& 
& Activation & ELU \\
& 
& MLP embedding dim & 128 \\
\tightcmidrule{2-4}
& \multirow{5}{*}{Depth Encoder}
& CNN Channels & [32, 64, 128] \\
& 
& CNN kernel sizes & [8, 4, 3] \\
& 
& CNN strides & [4, 2, 2] \\
& 
& Activation & ELU \\
& 
& CNN embedding dim & 128 \\
\tightcmidrule{2-4}
& \multirow{2}{*}{Temporal Encoder}
& GRU hidden dim & 256 \\
& 
& GRU layers & 1 \\
\tightcmidrule{2-4}
& \multirow{3}{*}{Fusion Encoder}
& MLP hidden size & [512, 256, 128] \\
& 
& Activation & ELU \\
& 
& MLP embedding dim & 48 \\
\tightcmidrule{2-4}
& \multirow{1}{*}{Latent Linear Heads}
& Output dims & $\hat{\mathbf z}_t^f: 16,\ \hat{\mathbf z}_t^b: 16,\ \hat{\mathbf z}_t^p: 16$ \\
\specialrule{0.4pt}{0pt}{0pt}

\multicolumn{2}{c}{\multirow{3}{*}{Estimator}}
& MLP hidden size & [512, 256, 128] \\
\multicolumn{2}{c}{}
& Activation & ELU \\
\multicolumn{2}{c}{}
& Output dims & $\hat{\mathbf v}_t: 3$ \\
\specialrule{0.4pt}{0pt}{0pt}

\multicolumn{2}{c}{\multirow{2}{*}{Foothold Imagination Model}}
& MLP hidden size & [256, 128] \\
\multicolumn{2}{c}{}
& Activation & ELU \\
\specialrule{0.4pt}{0pt}{0pt}

\multicolumn{2}{c}{\multirow{3}{*}{Multi-Discriminator}}
& Number of discriminators & 5 \\
\multicolumn{2}{c}{}
& MLP hidden size & [512, 256, 128] \\
\multicolumn{2}{c}{}
& Activation & ReLU \\
\specialrule{1.2pt}{0pt}{0pt}
\end{tabular}
\captionsetup{
    width=\linewidth,
    justification=centering,
    singlelinecheck=false,
    font=footnotesize,
    skip=4pt
}
\caption{Network architectures.}
\vspace{-14pt}
\label{table8}
\end{table}

\subsection{Training Hyperparameters}
\label{appendixA.5}

Table~\ref{table9} summarizes the main hyperparameters used to train the components of our framework.
We use four separate Adam optimizers for PPO, the encoder and estimator, the foothold imagination model, and the multi-discriminator, all initialized with a learning rate of $5 \times 10^{-4}$.

\begin{table}[H]
\centering
\vspace{5pt}
\renewcommand{\arraystretch}{1.25}
\scriptsize
\begin{tabular}{ccc}
\specialrule{1.2pt}{0pt}{1pt}
\textbf{Component} & \textbf{Parameter} & \textbf{Value} \\
\specialrule{0.8pt}{1pt}{0pt}
\multirow{10}{*}{PPO}
& Value loss coef & 1.0 \\
& Entropy coef & 0.005 \\
& Desired KL & 0.01 \\
& Minimum policy std & 0.2 \\
& GAE $\lambda$ & 0.95 \\
& Reward discount $\gamma$ & 0.99 \\
& Number of environments & 4096 \\
& Num steps per iteration& 24 \\
& Num learning epochs & 5 \\
& Num mini-batches & 4 \\
\specialrule{0.4pt}{0pt}{0pt}
\multirow{4}{*}{Encoder}
& Base height map MSE coef & 2.0 \\
& Foot height maps MSE coef & 1.0 \\
& Next proprioception MSE coef  & 5.0 \\
& VAE KL divergence coef & 1.0 \\
\specialrule{0.4pt}{0pt}{0pt}
Estimator
& Velocity MSE coef & 2.0 \\
\specialrule{0.4pt}{0pt}{0pt}
\multirow{2}{*}{Multi-Discriminator}
& Weight decay & $1 \times 10^{-4}$ \\
& Gradient penalty coef $w^{\mathrm{gp}}$ & 20.0 \\
\specialrule{1.2pt}{0pt}{0pt}
\end{tabular}
\captionsetup{
    width=\linewidth,
    justification=centering,
    singlelinecheck=false,
    font=footnotesize,
    skip=4pt
}
\caption{Training hyperparameters.}
\label{table9}
\end{table}

\subsection{Domain Randomization}
\label{appendixA.6}

Table~\ref{table10} summarizes the domain randomization settings used in our framework, including those applied to the system dynamics and depth sensing.

\begin{table}[H]
\centering
\vspace{5pt}
\renewcommand{\arraystretch}{1.25}
\footnotesize
\setlength{\tabcolsep}{5pt}
\begin{tabular}{ccc}
\specialrule{1.2pt}{0pt}{1pt}
\textbf{Category} & \textbf{Parameter} & \textbf{Range / Value} \\
\specialrule{0.8pt}{1pt}{0pt}

\multirow{10}{*}{\centering Dynamics}
& Base added mass & $\mathcal{U}(-2.0, 12.5)\,\mathrm{kg}$ \\
& Base CoM offset & $\mathcal{U}(-0.1, 0.1)\,\mathrm{m}$ \\
& Link mass & $\mathcal{U}(0.9, 1.1) \times \text{default}$ \\
& Link CoM offset & $\mathcal{U}(-0.01, 0.01)\,\mathrm{m}$ \\
& Motor strength & $\mathcal{U}(0.8, 1.2) \times \text{default}$ \\
& PD gains & $\mathcal{U}(0.9, 1.1) \times \text{default}$ \\
& Ground friction & $\mathcal{U}(0.2, 1.25)$ \\
& Action delay & $\mathcal{U}(0, 2)\,\text{control steps}$ \\
& Push interval & 12 s \\
& Push velocity & $\mathcal{U}(-0.4, 0.4)\,\mathrm{m/s}$ \\
\specialrule{0.4pt}{0pt}{0pt}

\multirow{5}{*}{\centering Depth Sensing}
& Camera position offset & $\mathcal{U}(-0.01, 0.01)\,\mathrm{m}$ \\
& Camera orientation offset & $\mathcal{U}(-1, 1)^\circ$ \\
& Horizontal FoV offset & $\mathcal{U}(-1, 1)^\circ$ \\
& Vertical FoV offset & $\mathcal{U}(-1, 1)^\circ$\\
& Self-occluded pixel noise & $\mathcal{N}(0, (0.1\mathrm{m})^2)$\\
\specialrule{1.2pt}{0pt}{0pt}
\end{tabular}
\captionsetup{
    width=\linewidth,
    justification=centering,
    singlelinecheck=false,
    font=footnotesize,
    skip=4pt
}
\caption{Domain randomization settings.}
\label{table10}
\end{table}

\subsection{Termination Criteria}
\label{appendixA.7}

An episode is terminated when any of the following conditions is satisfied:

\begin{itemize}[nosep,leftmargin=*]
\item The episode reaches its maximum duration.
\item The robot moves outside the terrain boundary.
\item The contact force on any termination-contact link exceeds a predefined threshold.
\item The robot exhibits severe orientation instability.
\item On gap terrain, any foot drops below a predefined threshold beneath the ground plane.
\end{itemize}

\subsection{Motion Prior Dataset}
\label{appendixA.8}

For each terrain type, the corresponding motion priors are sourced from two datasets: customized motion recordings collected using a motion capture system and a selected subset of the open-source AMASS dataset ~\cite{mahmood2019amass}. Across all terrain families, the source data contain 71 unique clips totaling 521 s, with cross-terrain reuse. All motion sequences are automatically retargeted to our humanoid following~\citet{xie2025kungfubot}. Table~\ref{table11} summarizes the statistics of the motion datasets used for training.

\begin{table}[H]
\centering
\vspace{5pt}
\setlength{\tabcolsep}{4pt}
\renewcommand{\arraystretch}{1.25}
\footnotesize
\resizebox{\columnwidth}{!}{
\begin{tabular}{cccccccccc}
\specialrule{1.2pt}{0pt}{1pt}
\multirow{2}{*}{\raisebox{-0.4ex}{\textbf{Terrain Type}}} & \multicolumn{3}{c}{\textbf{Flat-ground motions}} & \multicolumn{5}{c}{\textbf{Terrain traversals}} & \multirow{2}{*}{\raisebox{-0.4ex}{\textbf{Total Time (s)}}}\\
\noalign{\vskip -1.5pt}
\cmidrule[0.8pt](l{1.2pt}r{1.2pt}){2-4}
\cmidrule[0.8pt](l{1.2pt}r{1.2pt}){5-9}
\noalign{\vskip -1.5pt}
& \textbf{Forward} & \textbf{Backward} & \textbf{Turn} & \textbf{Slope} & \textbf{Stairs} & \textbf{Step} & \textbf{Gap} & \textbf{Platform} &  \\
\specialrule{0.8pt}{1pt}{0pt}
Slope    & \ding{51} & \ding{51} & \ding{51} & \ding{51} &  &  &  &  & 261.3 \\
Stairs   & \ding{51} &            & \ding{51} &            & \ding{51} &  &  &  & 295.8 \\
Discrete & \ding{51} &            & \ding{51} &            &            & \ding{51} &  &  & 136.0 \\
Gap      & \ding{51} &            & \ding{51} &            &            &            & \ding{51} &  & 174.5 \\
Platform & \ding{51} &            & \ding{51} &            &            & \ding{51} &            & \ding{51} & 181.6 \\
\specialrule{1.2pt}{0pt}{0pt}
\end{tabular}
}
\captionsetup{
    width=\linewidth,
    justification=centering,
    singlelinecheck=false,
    font=footnotesize,
    skip=4pt
}
\caption{Motion dataset statistics.}
\label{table11}
\end{table}


\section{Equivariant Network Details}
\label{AppendixB}

\subsection{Mirror Transformations and Symmetry-Structured Encoder Inputs}
\label{AppendixB.1}

This section formalizes the mirror transformations and symmetry operators used in our method. We first define the mirror transformations used for data augmentation in Sec.~\ref{sec3.3}. We then describe the structure of the encoder inputs and introduce the associated operators for constructing symmetry-structured representations, which are also used in the proofs in Appendix~\ref{AppendixB.2}.

We begin with the mirror transformation $\mathcal{M}_{c}$ for the latent representation $\hat{\mathbf{z}}_{t}$, defined as an exchange between the left and right channel groups. More generally, $\mathcal{M}_{c}$ is the primitive mirror operator for symmetry-structured vectors. For a vector whose channel dimension is partitioned into two paired groups, $\mathcal{M}_{c}$ swaps the two groups:
\begin{equation}
\setlength{\abovedisplayskip}{6pt}
\setlength{\belowdisplayskip}{-2pt}
\mathbf{X}
=
[\mathbf{X}^{(L)},\mathbf{X}^{(R)}]^\top
\qquad
\mathcal{M}_{c}\mathbf{X}
=
[\mathbf{X}^{(R)},\mathbf{X}^{(L)}]^\top.
\end{equation}

We next define the mirror transformation $\mathcal{M}_{p}$ for proprioception. To specify its action on joint-space variables, we introduce the following block notation. For any joint-space quantity, including joint positions, joint velocities, and actions, we write
$\mathbf{x}
\!=\!
\bigl[
\mathbf{x}_{R}^{\mathrm{leg}},\,
\mathbf{x}_{L}^{\mathrm{leg}},\,
x^{\mathrm{waist}},\,
\mathbf{x}_{R}^{\mathrm{arm}},\,
\mathbf{x}_{L}^{\mathrm{arm}}
\bigr]^{\top}
\!\in\! \mathbb{R}^{21}.$
Each leg block is ordered as \textit{(hip pitch, hip roll, hip yaw, knee, ankle pitch, ankle roll)}, and each arm block is ordered as \textit{(shoulder pitch, shoulder roll, shoulder yaw, elbow)}.

Under reflection about the sagittal plane, the left and right kinematic chains are exchanged, and the sign of each coordinate depends on the corresponding physical axis convention. We therefore define the sign vectors for the joint blocks as $\boldsymbol{s}_{\mathrm{leg}}=[1,-1,-1,1,1,-1]^{\top}$ and $\boldsymbol{s}_{\mathrm{arm}}=[1,-1,-1,1]^{\top}$. The resulting joint-block reflection is defined as
\begin{equation}
\setlength{\abovedisplayskip}{4pt}
\setlength{\belowdisplayskip}{4pt}
\mathcal{F}(\mathbf{x})
=
\bigl[
\mathbf{x}_{L}^{\mathrm{leg}}\odot\boldsymbol{s}_{\mathrm{leg}},\,
\mathbf{x}_{R}^{\mathrm{leg}}\odot\boldsymbol{s}_{\mathrm{leg}},\,
-x^{\mathrm{waist}},\,
\mathbf{x}_{L}^{\mathrm{arm}}\odot\boldsymbol{s}_{\mathrm{arm}},\,
\mathbf{x}_{R}^{\mathrm{arm}}\odot\boldsymbol{s}_{\mathrm{arm}}
\bigr]^{\top},
\end{equation}
where $\odot$ denotes element-wise multiplication. For the remaining proprioceptive quantities with explicit physical semantics, namely base-frame polar vectors, base angular velocities, and velocity commands, we use the sign vectors $\boldsymbol{s}_{y}=[1,-1,1]^{\top}$, $\boldsymbol{s}_{\omega}=[-1,1,-1]^{\top}$, and $\boldsymbol{s}_{u}=[1,-1,-1]^{\top}$.

Table~\ref{table12} summarizes the mirror transformations used for data augmentation in Sec.~\ref{sec3.3}, including those for proprioception, actor observations, privileged states, and actions.

\newcommand{\fullcmidrule}[4]{%
  \noalign{\vskip #2}%
  \cmidrule[#1](l{0pt}r{0pt}){#4}%
  \noalign{\vskip #3}%
}

\begin{table}[H]
  \centering
  \renewcommand{\arraystretch}{1.0}
  \scriptsize
  \resizebox{\columnwidth}{!}{%
  \begin{tabular}{ccccc}
      \specialrule{1.2pt}{0pt}{1pt}
      \textbf{Transformation} & \textbf{Component} & \textbf{Variable} & \textbf{Mirror rule} & \textbf{Dim.} \\
      \specialrule{0.8pt}{1pt}{0pt}

      \multirow{7}{*}{$\mathcal{M}_{p}$}
      & \multicolumn{4}{c}{\cellcolor{gray!15}\textbf{Proprioception $\mathbf{o}^{p}_{t}$}} \vspace{-0.3mm}\\
      \fullcmidrule{0.4pt}{-1pt}{-1pt}{2-5}
      & Base angular velocity & $\boldsymbol{\omega}_{t}$ & $\boldsymbol{\omega}_{t}\odot\boldsymbol{s}_{\omega}$ & $3$ \\
      & Projected gravity & $\mathbf{g}_{t}$ & $\mathbf{g}_{t}\odot\boldsymbol{s}_{y}$ & $3$ \\
      & Velocity command & $\mathbf{u}_{t}$ & $\mathbf{u}_{t}\odot\boldsymbol{s}_{u}$ & $3$ \\
      & Joint positions & $\boldsymbol{\theta}_{t}$ & $\mathcal{F}(\boldsymbol{\theta}_{t})$ & $21$ \\
      & Joint velocities & $\dot{\boldsymbol{\theta}}_{t}$ & $\mathcal{F}(\dot{\boldsymbol{\theta}}_{t})$ & $21$ \\
      & Previous action & $\mathbf{a}_{t-1}$ & $\mathcal{F}(\mathbf{a}_{t-1})$ & $21$ \\
      \specialrule{0.4pt}{0pt}{0pt}

      \multirow{4}{*}{\raisebox{-0.7ex}{$\mathcal{M}_{o}$}}
      & \multicolumn{4}{c}{\cellcolor{gray!15}\textbf{Actor observation $\mathbf{o}^{a}_{t}$}} \vspace{-0.3mm} \\
      \fullcmidrule{0.4pt}{-1pt}{-1pt}{2-5}
      & Proprioception & $\mathbf{o}^{p}_{t}$ & $\mathcal{M}_{p}\mathbf{o}^{p}_{t}$ & $72$ \\
      & Base velocity estimate & $\hat{\mathbf{v}}_{t}$ & $\mathcal{M}_{v}\hat{\mathbf{v}}_{t}=\hat{\mathbf{v}}_{t}\odot\boldsymbol{s}_{y}$ & $3$ \\
      & Encoder latent & $\hat {\mathbf{z}}_{t}=[\hat{\mathbf{z}}^{f}_{t},\hat{\mathbf{z}}^{b}_{t},\hat{\mathbf{z}}^{p}_{t}]^\top$ & $[\mathcal{M}_{c}\hat{\mathbf{z}}^{f}_{t},\mathcal{M}_{c}\hat{\mathbf{z}}^{b}_{t},\mathcal{M}_{c}\hat{\mathbf{z}}^{p}_{t}]^\top$ & $48$ \\
      \specialrule{0.4pt}{0pt}{0pt}

      \multirow{8}{*}{\raisebox{-4.0ex}{$\mathcal{M}_{s}$}}
      & \multicolumn{4}{c}{\cellcolor{gray!15}\textbf{Privileged state $\mathbf{s}_{t}$}} \vspace{-0.3mm}\\
      \fullcmidrule{0.4pt}{-1pt}{-1pt}{2-5}
      & Proprioception & $\mathbf{o}^{p}_{t}$ & $\mathcal{M}_{p}\mathbf{o}^{p}_{t}$ & $72$ \\
      & Base linear velocity & $\mathbf{v}_{t}$ & $\mathcal{M}_{v}{\mathbf{v}}_{t}=\mathbf{v}_{t}\odot\boldsymbol{s}_{y}$ & $3$ \\
      & Foot linear velocities & $[\mathbf{v}^{f}_{t,R},\,\mathbf{v}^{f}_{t,L}]^\top$ & $[\mathbf{v}^{f}_{t,L}\odot\boldsymbol{s}_{y},\,\mathbf{v}^{f}_{t,R}\odot\boldsymbol{s}_{y}]^\top$ & $6$ \\
      & Foot contact & $[c_{t,R},\,c_{t,L}]^\top$ & $[c_{t,L},\,c_{t,R}]^\top$ & $2$ \\
      & Hand positions & $[\mathbf{p}^{h}_{t,R},\,\mathbf{p}^{h}_{t,L}]^\top$ & $[\mathbf{p}^{h}_{t,L}\odot\boldsymbol{s}_{y},\,\mathbf{p}^{h}_{t,R}\odot\boldsymbol{s}_{y}]^\top$ & $6$ \\
      & Foot positions & $[\mathbf{p}^{f}_{t,R},\,\mathbf{p}^{f}_{t,L}]^\top$ & $[\mathbf{p}^{f}_{t,L}\odot\boldsymbol{s}_{y},\,\mathbf{p}^{f}_{t,R}\odot\boldsymbol{s}_{y}]^\top$ & $6$ \\
      & Base height map & $\mathbf{H}^{b}_{t}$ & $\mathcal{M}_{\mathrm{2D}}\mathbf{H}^{b}_{t}$ & $18\!\times\!10$ \\
      & Foot height maps & $(\mathbf{H}^{f}_{t,R},\,\mathbf{H}^{f}_{t,L})$ & $(\mathcal{M}_{\mathrm{2D}}\mathbf{H}^{f}_{t,L},\,\mathcal{M}_{\mathrm{2D}}\mathbf{H}^{f}_{t,R})$ & $2\!\times\!10\!\times\!5$ \\
      \specialrule{0.4pt}{0pt}{0pt}

      \multirow{2}{*}{\raisebox{-0.4ex}{$\mathcal{M}_{a}$}}
      & \multicolumn{4}{c}{\cellcolor{gray!15}\textbf{Action $\mathbf{a}_{t}$}} \vspace{-0.3mm}\\
      \fullcmidrule{0.4pt}{-1pt}{-1pt}{2-5}
      & Action & $\mathbf{a}_{t}$ & $\mathcal{F}(\mathbf{a}_{t})$ & $21$ \\
      \specialrule{1.2pt}{1pt}{0pt}
  \end{tabular}}
    \captionsetup{
        width=\linewidth,
        justification=centering,
        singlelinecheck=false,
        font=footnotesize,
        skip=4pt
    }
  \caption{Component-wise mirror rules used in the data augmentation procedure for proprioceptions, latent representations, actor observations, privileged states, and actions.}
  \label{table12}
\end{table}

Next, we analyze the structure of the encoder inputs and define the corresponding mirror operators used in the proofs in Appendix~\ref{AppendixB.2}. Because mirror equivariance is implemented through modules with explicitly paired left-right channels, their inputs must share the same symmetry-structured form. In our setting, this form is obtained by organizing the features into paired left and right channels. We describe this construction below for temporal proprioceptive and image inputs.

We first consider the temporal proprioception $\mathbf{o}^{p}_{t-h+1:t}$. This representation does not explicitly separate the left and right sides. We therefore introduce a fixed reorganization operator $\mathcal{T}$ that maps the length-$h$ proprioceptive history into a symmetry-structured left-right representation:
\begin{equation}
\bar{\mathbf{o}}^{p}_{t}
=
\mathcal{T}\bigl(\mathbf{o}^{p}_{t-h+1:t}\bigr)
=
[\bar{\mathbf{o}}^{p}_{t,L},\bar{\mathbf{o}}^{p}_{t,R}]^\top,
\qquad
\bar{\mathbf{o}}^{p}_{t,L},\,\bar{\mathbf{o}}^{p}_{t,R}\in\mathbb{R}^{42\times h}.
\end{equation}
Each branch is constructed by concatenating the corresponding per-frame slices over the past $h$ time steps. The definitions of the per-frame slices $\bar{\mathbf{o}}^{p}_{t-k,L}$ and $\bar{\mathbf{o}}^{p}_{t-k,R}$, for $k\in\{0,\ldots,h-1\}$, are listed in Table~\ref{table13}. Under this reorganization, the mirror transformation reduces to simple branch exchange:
\begin{equation}
\mathcal{M}_{c}\bar{\mathbf{o}}^{p}_{t}
=
[\bar{\mathbf{o}}^{p}_{t,R},\bar{\mathbf{o}}^{p}_{t,L}]^\top.
\end{equation}

\begin{table}[H]
  \centering
  \vspace{-5pt}
  \renewcommand{\arraystretch}{1.05}
  \scriptsize
  \resizebox{\columnwidth}{!}{%
  \begin{tabular}{cccc}
    \specialrule{1.2pt}{0pt}{1pt}
      \textbf{Component} & \textbf{Left slice $\bar{\mathbf{o}}^{p}_{t-k,L}$} & \textbf{Right slice $\bar{\mathbf{o}}^{p}_{t-k,R}$} & \textbf{Dim.}\\
      \specialrule{0.8pt}{1pt}{0pt}
      Base angular velocity 
      & $\boldsymbol{\omega}_{t-k}$ 
      & $\boldsymbol{\omega}_{t-k}\odot\boldsymbol{s}_{\omega}$ 
      & $3$\\
      Projected gravity 
      & $\mathbf{g}_{t-k}$ 
      & $\mathbf{g}_{t-k}\odot\boldsymbol{s}_{y}$ 
      & $3$\\
      Velocity command 
      & $\mathbf{u}_{t-k}$ 
      & $\mathbf{u}_{t-k}\odot\boldsymbol{s}_{u}$ 
      & $3$\\
      Joint positions 
      & $[\boldsymbol{\theta}^{\mathrm{leg}}_{t-k,L},\,\boldsymbol{\theta}^{\mathrm{arm}}_{t-k,L},\,\theta^{\mathrm{waist}}_{t-k}]^\top$ 
      & $[\boldsymbol{\theta}^{\mathrm{leg}}_{t-k,R}\odot\boldsymbol{s}_{\mathrm{leg}},\,\boldsymbol{\theta}^{\mathrm{arm}}_{t-k,R}\odot\boldsymbol{s}_{\mathrm{arm}},\,-\theta^{\mathrm{waist}}_{t-k}]^\top$ 
      & $11$\\
      Joint velocities
      & $[\dot{\boldsymbol{\theta}}^{\mathrm{leg}}_{t-k,L},\,\dot{\boldsymbol{\theta}}^{\mathrm{arm}}_{t-k,L},\,\dot{\theta}^{\mathrm{waist}}_{t-k}]^\top$ 
      & $[\dot{\boldsymbol{\theta}}^{\mathrm{leg}}_{t-k,R}\odot\boldsymbol{s}_{\mathrm{leg}},\,\dot{\boldsymbol{\theta}}^{\mathrm{arm}}_{t-k,R}\odot\boldsymbol{s}_{\mathrm{arm}},\,-\dot{\theta}^{\mathrm{waist}}_{t-k}]^\top$ 
      & $11$\\
      Previous action 
      & $[\mathbf{a}^{\mathrm{leg}}_{t-k-1,L},\,\mathbf{a}^{\mathrm{arm}}_{t-k-1,L},\,a^{\mathrm{waist}}_{t-k-1}]^\top$ 
      & $[\mathbf{a}^{\mathrm{leg}}_{t-k-1,R}\odot\boldsymbol{s}_{\mathrm{leg}},\,\mathbf{a}^{\mathrm{arm}}_{t-k-1,R}\odot\boldsymbol{s}_{\mathrm{arm}},\,-a^{\mathrm{waist}}_{t-k-1}]^\top$ 
      & $11$\\
      \specialrule{1.2pt}{0pt}{0pt}
  \end{tabular}}
  \captionsetup{
    width=\linewidth,
    justification=centering,
    singlelinecheck=false,
    font=footnotesize,
    skip=4pt
  }
  \caption{Per-frame left and right slices used by the reorganization operator $\mathcal{T}$.}
  \label{table13}
  \vspace{-15pt}
\end{table}

We then consider the image input $\mathbf{I}_{t}$. Images naturally exhibit a left-right spatial structure. We therefore use the horizontal reflection operator $\mathcal{M}_{\mathrm{2D}}$ for two-dimensional vectors. For $\mathbf{Y}\in\mathbb{R}^{H\times W}$, including an image or any image-derived representation, $\mathcal{M}_{\mathrm{2D}}$ acts along the width dimension and performs a horizontal flip:
\begin{equation}
\bigl[\mathcal{M}_{\mathrm{2D}}\mathbf{Y}\bigr]_{i,j}
=
\mathbf{Y}_{i,\,W+1-j},
\qquad
i=1,\ldots,H,\;\; j=1,\ldots,W.
\end{equation}

Although an image already has a left-right spatial structure, it does not directly provide paired left-right channels along the channel dimension. To address this issue, we introduce a lift convolution layer in the equivariant CNN (Appendix~\ref{AppendixB.2}), which lifts the input image to a symmetry-structured feature representation with explicitly paired left and right channels.

\subsection{Equivariance Proofs for Network Components}
\label{AppendixB.2}

To construct an encoder that satisfies the desired equivariance property, we design each module as an equivariant network. These modules include equivariant MLPs, equivariant GRUs, and an equivariant CNN, which can in turn be decomposed into compositions of equivariant linear and convolutional layers. In this section, we first present the construction and properties of these equivariant network components and provide the corresponding proofs. We then use these results to establish the equivariance of each latent head in the encoder of Sec.~\ref{sec3.3}.

\paragraph{Equivariant Linear Layer}

An equivariant linear layer has the same affine form as a standard linear layer, but its weight matrix and bias are constrained to follow a symmetry-preserving block structure. We define $f:\mathbb{R}^{2C_{\mathrm{in}}}\rightarrow\mathbb{R}^{2C_{\mathrm{out}}}$ as
\begin{equation}
f(\mathbf{x})
=
W\mathbf{x}+\mathbf{b},
\qquad
\mathbf{x}
=
\begin{bmatrix}
\mathbf{x}^{(1)}\\
\mathbf{x}^{(2)}
\end{bmatrix},
\qquad
W=
\begin{bmatrix}
A & B\\
B & A
\end{bmatrix},
\qquad
\mathbf{b}
=
\begin{bmatrix}
\mathbf{a}\\
\mathbf{a}
\end{bmatrix}.
\end{equation}
Here, $\mathbf{x}\in\mathbb{R}^{2C_{\mathrm{in}}}$ is the symmetry-structured input feature vector, $A,B\in\mathbb{R}^{C_{\mathrm{out}}\times C_{\mathrm{in}}}$ are learnable weight blocks, and $\mathbf{a}\in\mathbb{R}^{C_{\mathrm{out}}}$ is the shared bias.

\begin{proposition}
\label{prop1}
The linear layer $f$ is $\mathcal{M}_c$-equivariant:
\begin{equation}
f(\mathcal{M}_{c}\mathbf{x})
=
\mathcal{M}_{c}f(\mathbf{x}).
\end{equation}
\end{proposition}

\begin{proof}
Direct substitution yields
\[
f(\mathcal{M}_{c}\mathbf{x})
=
\begin{bmatrix}
A\mathbf{x}^{(2)}+B\mathbf{x}^{(1)}+\mathbf{a}\\
B\mathbf{x}^{(2)}+A\mathbf{x}^{(1)}+\mathbf{a}
\end{bmatrix}
=
\mathcal{M}_{c}
\begin{bmatrix}
A\mathbf{x}^{(1)}+B\mathbf{x}^{(2)}+\mathbf{a}\\
B\mathbf{x}^{(1)}+A\mathbf{x}^{(2)}+\mathbf{a}
\end{bmatrix}
=
\mathcal{M}_{c}f(\mathbf{x}).
\]
\end{proof}

\paragraph{Equivariant MLP}

The equivariant MLP follows the standard MLP form, except that each linear layer is replaced with an equivariant linear layer $f_{\ell}$:
\begin{equation}
F_{\mathrm{MLP}}(\mathbf{x})
=
f_{L}
\Bigl(
\sigma
\bigl(
f_{L-1}
(
\cdots
\sigma(f_{1}(\mathbf{x}))
\cdots
)
\bigr)
\Bigr),
\end{equation}
where $\sigma$ denotes the activation function.

\begin{proposition}
\label{prop2}
The equivariant MLP $F_{\mathrm{MLP}}$ is $\mathcal{M}_c$-equivariant:
\begin{equation}
F_{\mathrm{MLP}}(\mathcal{M}_{c}\mathbf{x})
=
\mathcal{M}_{c}F_{\mathrm{MLP}}(\mathbf{x}).
\end{equation}
\end{proposition}

\begin{proof}
Let
\[
G_\ell =
\begin{cases}
\sigma \circ f_\ell, & \ell=1,\dots,L-1,\\
f_L, & \ell=L.
\end{cases}
\]
By Proposition~\ref{prop1}, each equivariant linear layer $f_\ell$ commutes with $\mathcal{M}_c$. Since $\sigma$ is applied element-wise, it also commutes with $\mathcal{M}_c$. Hence every $G_\ell$ is $\mathcal{M}_c$-equivariant. Because a composition of equivariant maps is equivariant,
\[
F_{\mathrm{MLP}}
=
G_L \circ G_{L-1} \circ \cdots \circ G_1
\]
also satisfies
\[
F_{\mathrm{MLP}}(\mathcal{M}_{c}\mathbf{x})
=
\mathcal{M}_{c}F_{\mathrm{MLP}}(\mathbf{x}).
\]
\end{proof}

\paragraph{Equivariant GRU}

The equivariant GRU follows the standard update rule, with $\mathbf{h}_t=F_{\mathrm{GRU}}(\mathbf{x}_t,\mathbf{h}_{t-1})$, where all affine transformations use the same block-structured parameterization as the equivariant linear layer:
\begin{equation}
\left\{
\begin{aligned}
\mathbf{r}_t
&=
\sigma
\left(
W_{\mathrm{ir}}\mathbf{x}_t+\mathbf{b}_{\mathrm{ir}}
+
W_{\mathrm{hr}}\mathbf{h}_{t-1}+\mathbf{b}_{\mathrm{hr}}
\right),\\
\mathbf{z}_t
&=
\sigma
\left(
W_{\mathrm{iz}}\mathbf{x}_t+\mathbf{b}_{\mathrm{iz}}
+
W_{\mathrm{hz}}\mathbf{h}_{t-1}+\mathbf{b}_{\mathrm{hz}}
\right),\\
\mathbf{n}_t
&=
\tanh
\left(
W_{\mathrm{in}}\mathbf{x}_t+\mathbf{b}_{\mathrm{in}}
+
\mathbf{r}_t\odot
\left(
W_{\mathrm{hn}}\mathbf{h}_{t-1}+\mathbf{b}_{\mathrm{hn}}
\right)
\right),\\
\mathbf{h}_t
&=
(\mathbf{1}-\mathbf{z}_t)\odot\mathbf{n}_t
+
\mathbf{z}_t\odot\mathbf{h}_{t-1}.
\end{aligned}
\right.
\end{equation}
Here, $\mathbf{x}_t\in\mathbb{R}^{2C_{\mathrm{in}}}$ denotes the input feature at time $t$; $\mathbf{h}_{t-1}$ and $\mathbf{h}_t$ are the previous and updated hidden states in $\mathbb{R}^{2C_{\mathrm{h}}}$, and $\mathbf{r}_t$, $\mathbf{z}_t$, and $\mathbf{n}_t$ denote the reset gate, update gate, and candidate hidden state, respectively.

\begin{proposition}
\label{prop3}
The one-step GRU update is equivariant under channel exchange:
\begin{equation}
F_{\mathrm{GRU}}
\bigl(
\mathcal{M}_{c}\mathbf{x}_t,
\mathcal{M}_{c}\mathbf{h}_{t-1}
\bigr)
=
\mathcal{M}_{c}
F_{\mathrm{GRU}}(\mathbf{x}_t,\mathbf{h}_{t-1})
.
\end{equation}
\end{proposition}

\begin{proof}
By Proposition~\ref{prop1}, all affine maps in the GRU update commute with $\mathcal{M}_c$. Since $\sigma$, $\tanh$, subtraction from $\mathbf{1}$, and the Hadamard product are all pointwise operations, they also commute with $\mathcal{M}_c$. Therefore
\[
\mathbf{r}'_t=\mathcal{M}_{c}\mathbf{r}_t,
\qquad
\mathbf{z}'_t=\mathcal{M}_{c}\mathbf{z}_t.
\]
For the candidate hidden state,
\[
\begin{aligned}
\mathbf{n}'_t
&=
\tanh\!\Big(
W_{\mathrm{in}}\mathcal{M}_{c}\mathbf{x}_t+\mathbf{b}_{\mathrm{in}}
+
\mathbf{r}'_t\odot
\bigl(
W_{\mathrm{hn}}\mathcal{M}_{c}\mathbf{h}_{t-1}+\mathbf{b}_{\mathrm{hn}}
\bigr)
\Big) \\
&=
\tanh\!\Big(
\mathcal{M}_{c}
\bigl(
W_{\mathrm{in}}\mathbf{x}_t+\mathbf{b}_{\mathrm{in}}
+
\mathbf{r}_t\odot(W_{\mathrm{hn}}\mathbf{h}_{t-1}+\mathbf{b}_{\mathrm{hn}})
\bigr)
\Big)
=
\mathcal{M}_{c}\mathbf{n}_t.
\end{aligned}
\]
Hence
\[
\begin{aligned}
F_{\mathrm{GRU}}(\mathcal{M}_{c}\mathbf{x}_t,\mathcal{M}_{c}\mathbf{h}_{t-1})
&=
(\mathbf{1}-\mathbf{z}'_t)\odot\mathbf{n}'_t
+
\mathbf{z}'_t\odot\mathcal{M}_{c}\mathbf{h}_{t-1} \\
&=
\mathcal{M}_{c}
\bigl(
(\mathbf{1}-\mathbf{z}_t)\odot\mathbf{n}_t
+
\mathbf{z}_t\odot\mathbf{h}_{t-1}
\bigr)
=
\mathcal{M}_{c}F_{\mathrm{GRU}}(\mathbf{x}_t,\mathbf{h}_{t-1}).
\end{aligned}
\]
\end{proof}

\paragraph{Equivariant Convolution Layer}
We define the equivariant convolutional layers using cross-correlation and symmetry-constrained kernels. For an input feature map $\mathbf{x}\in\mathbb{R}^{C_{\mathrm{in}}\times H\times W}$, let $\bar{\mathbf{x}}\in\mathbb{R}^{C_{\mathrm{in}}\times \bar{H}\times \bar{W}}$ denote its padded version, where $\bar{H}=H+2p_H$ and $\bar{W}=W+2p_W$. 
Given a convolution kernel $\mathbf{k}\in\mathbb{R}^{C_{\mathrm{out}}\times C_{\mathrm{in}}\times K_H\times K_W}$, we define cross-correlation with stride $s$ as
\begin{equation}
(\mathbf{x}\star\mathbf{k})_{o,h,w}
=
\sum_{i=1}^{C_{\mathrm{in}}}\sum_{u=1}^{K_H}\sum_{v=1}^{K_W}
\mathbf{k}_{o,i,u,v}\,
\bar{\mathbf{x}}_{i,(h-1)s+u,(w-1)s+v}.
\end{equation}
Here, $h\!=\!1,\ldots,H_o$ and $w\!=\!1,\ldots,W_o$, where
$H_o\!=\!\lfloor(\bar{H}-K_H)/s\rfloor\!+\!1$ and
$W_o\!=\!\lfloor(\bar{W}-K_W)/s\rfloor\!+\!1$.
When $s\!>\!1$, we assume $\bar{W}\!-\!K_W$ is divisible by $s$ so that horizontal sampling grid remains aligned under flipping. To describe mirror equivariance in the spatial domain, we define the horizontal flip operator on padded feature maps and kernels as
\begin{equation}
[\mathcal{M}_{\mathrm{2D}}\bar{\mathbf{x}}]_{i,h,w}
=
\bar{\mathbf{x}}_{i,h,\bar{W}+1-w},
\qquad
[\mathcal{M}_{\mathrm{2D}}\mathbf{k}]_{o,i,u,v}
=
\mathbf{k}_{o,i,u,K_W+1-v}.
\end{equation}

An ordinary image does not explicitly contain paired symmetry channels. We therefore use a lift convolution to map it to a symmetry-structured representation:
\begin{equation}
f_{\mathrm{l}}(\mathbf{x}) = \mathbf{x}\star K_{\mathrm{l}} + \mathbf{b}_{\mathrm{l}},
\qquad
K_{\mathrm{l}} =
\begin{bmatrix}
\mathbf{k}_{\mathrm{l}} \\
\mathcal{M}_{\mathrm{2D}}\mathbf{k}_{\mathrm{l}}
\end{bmatrix},
\qquad
\mathbf{b}_{\mathrm{l}} =
\begin{bmatrix}
\mathbf{a}_{\mathrm{l}} \\
\mathbf{a}_{\mathrm{l}}
\end{bmatrix}.
\end{equation}
Here, $\mathbf{k}_{\mathrm{l}}\in\mathbb{R}^{C_{\mathrm{out}}\times C_{\mathrm{in}}\times K_H\times K_W}$ is a learnable kernel, and $\mathbf{a}_{\mathrm{l}}\in\mathbb{R}^{C_{\mathrm{out}}}$ is a shared bias. After lifting, the subsequent layers preserve the symmetry-structured representation through constrained kernels. Each non-lift convolution is defined as
\begin{equation}
f_{\mathrm{nl}}(\mathbf{x}) = \mathbf{x}\star K_{\mathrm{nl}} + \mathbf{b}_{\mathrm{nl}},
\quad
\mathbf{x}
=
\begin{bmatrix}
\mathbf{x}^{(1)}\\
\mathbf{x}^{(2)}
\end{bmatrix},
\quad
K_{\mathrm{nl}}
=
\begin{bmatrix}
\mathbf{k}_{\mathrm{a}} & \mathbf{k}_{\mathrm{b}} \\
\mathcal{M}_{\mathrm{2D}}\mathbf{k}_{\mathrm{b}} & \mathcal{M}_{\mathrm{2D}}\mathbf{k}_{\mathrm{a}}
\end{bmatrix},
\quad
\mathbf{b}_{\mathrm{nl}} =
\begin{bmatrix}
\mathbf{a}_{\mathrm{nl}} \\
\mathbf{a}_{\mathrm{nl}}
\end{bmatrix}.
\end{equation}
Here, $\mathbf{x}^{(1)},\mathbf{x}^{(2)}\in\mathbb{R}^{C_{\mathrm{in}}\times H\times W}$ are the two symmetry branches, 
$\mathbf{k}_{\mathrm{a}},\mathbf{k}_{\mathrm{b}}\in\mathbb{R}^{C_{\mathrm{out}}\times C_{\mathrm{in}}\times K_H\times K_W}$ are learnable kernels, and 
$\mathbf{a}_{\mathrm{nl}}\in\mathbb{R}^{C_{\mathrm{out}}}$ is a shared bias.

\begin{lemma}
\label{lem1}
Cross-correlation satisfies
\begin{equation}
(\mathcal{M}_{\mathrm{2D}}\mathbf{x})\star\mathbf{k}
=
\mathcal{M}_{\mathrm{2D}}
\bigl(
\mathbf{x}\star(\mathcal{M}_{\mathrm{2D}}\mathbf{k})
\bigr).
\label{eq:flip_corr_identity}
\end{equation}
\end{lemma}

\begin{proof}
For the left-hand side,
\[
[(\mathcal{M}_{\mathrm{2D}}\mathbf{x})\star\mathbf{k}]_{o,h,w}
=
\sum_{i,u,v}
\mathbf{k}_{o,i,u,v}\,
\bar{\mathbf{x}}_{i,(h-1)s+u,\bar{W}+1-(w-1)s-v}.
\]
For the right-hand side, substituting $v'=K_W+1-v$ gives
\[
[\mathcal{M}_{\mathrm{2D}}(\mathbf{x}\star(\mathcal{M}_{\mathrm{2D}}\mathbf{k}))]_{o,h,w}
=
\sum_{i,u,v'}
\mathbf{k}_{o,i,u,v'}\,
\bar{\mathbf{x}}_{i,(h-1)s+u,(W_o-w)s+K_W+1-v'}.
\]
Using $W_o=\frac{\bar{W}-K_W}{s}+1$ and the alignment assumption, the last index equals
\[
\bar{W}+1-(w-1)s-v',
\]
so the two expressions coincide.
\end{proof}

\begin{proposition}
\label{prop4}
The lift convolution satisfies
\begin{equation}
f_{\mathrm{l}}(\mathcal{M}_{\mathrm{2D}}\mathbf{x})
=
\mathcal{M}_{c}
\mathcal{M}_{\mathrm{2D}}(f_{\mathrm{l}}(\mathbf{x}))
.
\end{equation}
\end{proposition}

\begin{proof}
By Lemma~\ref{lem1},
\[
(\mathcal{M}_{\mathrm{2D}}\mathbf{x})\star\mathbf{k}_{\mathrm{l}}
=
\mathcal{M}_{\mathrm{2D}}
\bigl(
\mathbf{x}\star(\mathcal{M}_{\mathrm{2D}}\mathbf{k}_{\mathrm{l}})
\bigr),
\qquad
(\mathcal{M}_{\mathrm{2D}}\mathbf{x})\star(\mathcal{M}_{\mathrm{2D}}\mathbf{k}_{\mathrm{l}})
=
\mathcal{M}_{\mathrm{2D}}(\mathbf{x}\star\mathbf{k}_{\mathrm{l}}).
\]
Hence
\[
\begin{aligned}
f_{\mathrm{l}}(\mathcal{M}_{\mathrm{2D}}\mathbf{x})
&=
\begin{bmatrix}
\mathcal{M}_{\mathrm{2D}}
\bigl(\mathbf{x}\star(\mathcal{M}_{\mathrm{2D}}\mathbf{k}_{\mathrm{l}})+\mathbf{a}_{\mathrm{l}}\bigr)\\
\mathcal{M}_{\mathrm{2D}}
\bigl(\mathbf{x}\star\mathbf{k}_{\mathrm{l}}+\mathbf{a}_{\mathrm{l}}\bigr)
\end{bmatrix} 
=
\mathcal{M}_{c}\mathcal{M}_{\mathrm{2D}}f_{\mathrm{l}}(\mathbf{x}).
\end{aligned}
\]
\end{proof}

\begin{proposition}
\label{prop5}
The non-lift convolution satisfies
\begin{equation}
f_{\mathrm{nl}}
\bigl(
\mathcal{M}_{c}\mathcal{M}_{\mathrm{2D}}\mathbf{x}
\bigr)
=
\mathcal{M}_{c}
\mathcal{M}_{\mathrm{2D}}f_{\mathrm{nl}}(\mathbf{x})
.
\end{equation}
\end{proposition}

\begin{proof}
Let
\[
y_1
=
\mathbf{x}^{(1)}\star\mathbf{k}_{\mathrm{a}}
+
\mathbf{x}^{(2)}\star\mathbf{k}_{\mathrm{b}}
+
\mathbf{a}_{\mathrm{nl}},
\qquad
y_2
=
\mathbf{x}^{(1)}\star(\mathcal{M}_{\mathrm{2D}}\mathbf{k}_{\mathrm{b}})
+
\mathbf{x}^{(2)}\star(\mathcal{M}_{\mathrm{2D}}\mathbf{k}_{\mathrm{a}})
+
\mathbf{a}_{\mathrm{nl}}.
\]
Then
\[
f_{\mathrm{nl}}(\mathbf{x})
=
\begin{bmatrix}
y_1\\
y_2
\end{bmatrix}.
\]
Using Lemma~\ref{lem1},
\[
f_{\mathrm{nl}}(\mathcal{M}_{c}\mathcal{M}_{\mathrm{2D}}\mathbf{x})
=
\begin{bmatrix}
\mathcal{M}_{\mathrm{2D}}y_2\\
\mathcal{M}_{\mathrm{2D}}y_1
\end{bmatrix}
=
\mathcal{M}_{c}\mathcal{M}_{\mathrm{2D}}f_{\mathrm{nl}}(\mathbf{x}).
\]
\end{proof}

\paragraph{Equivariant CNN}
The equivariant CNN is constructed by composing the layer-wise equivariant convolutions defined above. It consists of one lift convolution $f_{\mathrm{l}}$ followed by $N-1$ non-lift convolutions $f_{\mathrm{nl}}^{i}$:
\begin{equation}
F_{\mathrm{CNN}}(\mathbf{x})
=
f_{\mathrm{nl}}^{N-1}
\Bigl(
\sigma
\bigl(
f_{\mathrm{nl}}^{N-2}
(
\cdots
\sigma
(
f_{\mathrm{nl}}^{1}
(
\sigma(f_{\mathrm{l}}(\mathbf{x}))
)
)
\cdots
)
\bigr)
\Bigr).
\end{equation}
Let $F_{\mathrm{CNN}}(\mathbf{x})\in\mathbb{R}^{2C_N\times H_N\times W_N}$ denote the output feature map of the equivariant CNN, where $H_N$ and $W_N$ are its spatial height and width.

\begin{proposition}
\label{prop6}
The equivariant CNN satisfies
\begin{equation}
F_{\mathrm{CNN}}(\mathcal{M}_{\mathrm{2D}}\mathbf{x})
=
\mathcal{M}_{c}
\mathcal{M}_{\mathrm{2D}}F_{\mathrm{CNN}}(\mathbf{x})
.
\end{equation}
\end{proposition}

\begin{proof}
Let
\[
T := \mathcal{M}_{c}\mathcal{M}_{\mathrm{2D}}.
\]
By Proposition~\ref{prop4}, the lift convolution is $T$-equivariant. By Proposition~\ref{prop5}, each non-lift convolution is also $T$-equivariant. Since $\sigma$ is applied element-wise, it commutes with both channel exchange and horizontal flipping, and hence with $T$. Therefore every layer in the composition defining $F_{\mathrm{CNN}}$ is $T$-equivariant. By closure under composition,
\[
F_{\mathrm{CNN}}(\mathcal{M}_{\mathrm{2D}}\mathbf{x})
=
T F_{\mathrm{CNN}}(\mathbf{x})
=
\mathcal{M}_{c}\mathcal{M}_{\mathrm{2D}}F_{\mathrm{CNN}}(\mathbf{x}).
\]
\end{proof}

\vspace{-10pt}
\begin{corollary}
\label{cor1}
If the output feature map has spatial resolution $1\times1$, then
\begin{equation}
F_{\mathrm{CNN}}(\mathcal{M}_{\mathrm{2D}}\mathbf{x})
=
\mathcal{M}_{c}
F_{\mathrm{CNN}}(\mathbf{x})
.
\end{equation}
\end{corollary}
\vspace{10pt}

\begin{theorem}
\label{thm1}
Let $E_{\phi}^{(k)},k\in\{f,b,p\}$ denote the mapping from the symmetry-structured proprioception $\bar{\mathbf{o}}^p_t$ and the image $\mathbf{I}_t$ to the $k$-th latent head $\hat{\mathbf{z}}_t^{\,k}$ in the encoder used in Sec.~\ref{sec3.3}. Then each latent head is equivariant under the joint action of channel exchange on proprioception and horizontal reflection on images:
\begin{equation}
E_{\phi}^{(k)}(\mathcal{M}_{c}\bar{\mathbf{o}}^p_{t},\mathcal{M}_{\mathrm{2D}}\mathbf{I}_t)
=
\mathcal{M}_{c}E_{\phi}^{(k)}(\bar{\mathbf{o}}^p_{t},\mathbf{I}_t),
\end{equation}
where $\mathcal{M}_{c}$ denotes the mirror operator acting on the $k$-th latent head.
\end{theorem}

\begin{proof}
The encoder in Sec.~\ref{sec3.3} consists of an equivariant proprioceptive MLP, an equivariant visual CNN, a GRU applied to their concatenated features, a fusion MLP, and the $k$-th latent head.

First, the proprioceptive branch is implemented by equivariant MLPs. By Proposition~\ref{prop2}, each such MLP is $\mathcal{M}_{c}$-equivariant. Hence the proprioceptive feature transforms equivariantly under channel exchange.

Second, the visual branch is implemented by an equivariant CNN. By Proposition~\ref{prop6}, the CNN satisfies
\[
F_{\mathrm{CNN}}(\mathcal{M}_{\mathrm{2D}}\mathbf{x})
=
\mathcal{M}_{c}\mathcal{M}_{\mathrm{2D}}F_{\mathrm{CNN}}(\mathbf{x}).
\]
For the architecture used in this work, the output spatial resolution is $1\times 1$. Hence, by Corollary~\ref{cor1}, the visual feature produced by the CNN satisfies
\[
F_{\mathrm{CNN}}(\mathcal{M}_{\mathrm{2D}}\mathbf{I}_t)
=
\mathcal{M}_{c}F_{\mathrm{CNN}}(\mathbf{I}_t).
\]

Next, let $\oplus$ denote feature concatenation along the symmetry-structured channel dimension. Since both proprioceptive and visual features transform under the same channel-exchange operator, concatenation preserves equivariance:
\[
(\mathcal{M}_{c}\mathbf{u}) \oplus (\mathcal{M}_{c}\mathbf{v})
=
\mathcal{M}_{c}(\mathbf{u}\oplus\mathbf{v}).
\]

The concatenated feature is then processed by the GRU. By Proposition~\ref{prop3}, the GRU is $\mathcal{M}_{c}$-equivariant. Its output is further mapped by the fusion MLP, which is also $\mathcal{M}_{c}$-equivariant by Proposition~\ref{prop2}$.$

Finally, the fused feature is mapped to the $k$-th latent head by an equivariant head mapping. By Proposition~\ref{prop2}, this mapping is equivariant with respect to $\mathcal{M}_{c}$ on both its input and output. Combining the equivariance of the proprioceptive MLP, the visual CNN, the feature concatenation, the GRU, the fusion MLP, and the final head yields
\[
E_{\phi}^{(k)}(\mathcal{M}_{c}\bar{\mathbf{o}}^p_{t},\mathcal{M}_{\mathrm{2D}}\mathbf{I}_t)
=
\mathcal{M}_{c}E_{\phi}^{(k)}(\bar{\mathbf{o}}^p_{t},\mathbf{I}_t).
\]
This completes the proof.
\end{proof}


\section{Efficient Depth Rendering with Self-Occlusion}
\label{appendixC}

To synthesize depth observations consistent with deployment-time sensing, we ray-cast against both the static terrain mesh and the robot's dynamic link meshes. A naive implementation would update dynamic mesh poses and repeatedly refit the corresponding ray-query acceleration structures, such as bounding volume hierarchies (BVHs), at every simulation step, which quickly becomes a bottleneck in massively parallel training. To avoid this overhead, we keep each mesh in its local frame and instead transform each camera ray from the world frame into the queried mesh frame before ray-mesh intersection. Because these transformations are rigid, hit distances remain directly comparable across meshes. This yields an efficient depth renderer with self-occlusion, implemented in NVIDIA Warp~\cite{macklin2022warp}, for large-scale GPU-parallel simulation. The procedure is summarized in Algorithm~\ref{alg1}.

\begin{table}[H]
\centering
\vspace{5pt}
\footnotesize
\refstepcounter{algorithm}
\label{alg1}
\renewcommand{\arraystretch}{1.2}
\scriptsize
\begin{tabular}{p{0.95\linewidth}}
\specialrule{1.2pt}{0.0pt}{1.0pt}
\textbf{Algorithm~\thealgorithm: Efficient Depth Rendering with Self-Occlusion} \\
\specialrule{0.8pt}{1.0pt}{1.0pt}

\textbf{Require:} terrain mesh $\mathcal{M}^{\mathrm{ter}}$ with pose $(\mathbf{p}^{\mathrm{ter}},\mathbf{R}^{\mathrm{ter}})$; per-link meshes $\{\mathcal{M}^{\ell}_{j}\}_{j=1}^{N^{\ell}}$ with per-env link poses $\{(\mathbf{p}^{\ell}_{j,e},\mathbf{R}^{\ell}_{j,e})\}$; per-env camera poses $\{(\mathbf{p}^{\mathrm{cam}}_e,\mathbf{R}^{\mathrm{cam}}_e)\}_{e=1}^{N^{\mathrm{env}}}$; half-FoV tangents $(\tan_{x,e},\tan_{y,e})$; depth bounds $(d_{\min},d_{\max})$.\\

\textbf{Ensure:} clipped camera-depth maps $\mathbf{D} \in \mathbb{R}^{N^{\mathrm{env}} \times H \times W}$.\\[1mm]

\quad $\triangleright$ \textit{Parallel mesh-based occluder rendering with one thread per pixel}\\[0.5mm]

\begin{tabular}{@{}r@{\quad}p{0.92\linewidth}@{}}
1: & \textbf{for each} pixel $(e,h,w) \in [N^{\mathrm{env}}] \times [H] \times [W]$ \textbf{in parallel do}\\
2: & \quad Image-plane ray coordinates:
$t_x \leftarrow \bigl(2(w-\frac{1}{2})/W - 1\bigr)\tan_{x,e},
t_y \leftarrow \bigl(2(h-\frac{1}{2})/H - 1\bigr)\tan_{y,e}$\\
3: & \quad World frame ray origin:
$\mathbf{o} \leftarrow \mathbf{p}^{\mathrm{cam}}_e$\\
4: & \quad World frame ray direction:
$\mathbf{d} \leftarrow \operatorname{Normalize}\bigl(\operatorname{Rot}(\mathbf{R}^{\mathrm{cam}}_e,[t_x,t_y,-1]^\top)\bigr)$\\[0.5mm]

5: & \quad Terrain frame ray:
$(\mathbf{o}^{\mathrm{ter}},\mathbf{d}^{\mathrm{ter}}) \leftarrow
\operatorname{Transform}^{-1}\bigl((\mathbf{p}^{\mathrm{ter}},\mathbf{R}^{\mathrm{ter}}),(\mathbf{o},\mathbf{d})\bigr)$\\
6: & \quad Terrain hit distance:
$t^{\mathrm{ter}} \leftarrow
\mathrm{RayQuery}(\mathcal{M}^{\mathrm{ter}},\mathbf{o}^{\mathrm{ter}},\mathbf{d}^{\mathrm{ter}})$\\[0.5mm]

7: & \quad Initialize nearest self-occlusion distance:
$t^{\mathrm{occ}} \leftarrow +\infty$\\
8: & \quad \textbf{for} $j = 1$ \textbf{to} $N^{\ell}$ \textbf{do}\\
9: & \qquad Link frame ray:
$(\mathbf{o}^{\ell}_j,\mathbf{d}^{\ell}_j) \leftarrow
\mathrm{Transform}^{-1}\bigl((\mathbf{p}^{\ell}_{j,e},\mathbf{R}^{\ell}_{j,e}),(\mathbf{o},\mathbf{d})\bigr)$\\
10: & \qquad Link hit distance:
$t^{\ell}_j \leftarrow
\operatorname{RayQuery}(\mathcal{M}^{\ell}_{j},\mathbf{o}^{\ell}_j,\mathbf{d}^{\ell}_j)$\\
11: & \qquad $t^{\mathrm{occ}} \leftarrow \min(t^{\mathrm{occ}},t^{\ell}_j)$\\
12: & \qquad \textbf{if} $t^{\mathrm{occ}} \le d_{\min}$ \textbf{then break}\\
13: & \quad \textbf{end for}\\[0.5mm]

14: & \quad Convert ray distance to camera $z$-depth:
$c \leftarrow (1+t_x^2+t_y^2)^{-1/2}$\\
15: & \quad $\mathbf{D}[e,h,w] \leftarrow
\operatorname{clamp}\bigl(c \cdot \min(t^{\mathrm{ter}},t^{\mathrm{occ}}),d_{\min},d_{\max}\bigr)$\\
16: & \textbf{end for}\\
17: & \textbf{return} $\mathbf{D}$\\
\end{tabular}\\

\specialrule{1.2pt}{1.0pt}{0pt}
\end{tabular}
\end{table}


\newpage
\section{Real-World Deployment Details}
\label{AppendixD}

\setlength{\columnsep}{6pt}
\begin{wraptable}{r}{0.36\textwidth}
\vspace{-10pt}
\begin{minipage}{0.36\textwidth}
\centering
\renewcommand{\arraystretch}{1.1}
\setlength{\tabcolsep}{1.8pt}
\footnotesize
\begin{tabular}{cccc}
\specialrule{1.2pt}{0pt}{1pt}
\textbf{Joint Name} & \textbf{$K_p$} & \textbf{$K_d$} & \textbf{Action Scale} \\
\specialrule{0.8pt}{1pt}{0pt}
Hip pitch & 120 & 4 & 0.25 \\
Hip yaw & 100 & 2 & 0.25 \\
Hip roll & 100 & 2 & 0.25 \\
Knee & 120 & 4 & 0.25 \\
Ankle pitch & 40 & 2.0 & 0.25 \\
Ankle roll & 20 & 1.0 & 0.25 \\
Waist yaw & 100 & 3 & 0.2 \\
Shoulder pitch & 30 & 1 & 0.2 \\
Shoulder roll & 30 & 1 & 0.2 \\
Shoulder yaw & 30 & 1 & 0.2 \\
Elbow & 30 & 1 & 0.2 \\
\specialrule{1.2pt}{0pt}{0pt}
\end{tabular}
\captionsetup{
    width=\linewidth,
    singlelinecheck=false,
    skip=3pt,
    font=footnotesize
}
\caption{Joint-wise deployment gains and action scales.}
\label{table14}
\end{minipage}
\vspace{-16pt}
\end{wraptable}

We deploy our policy on the AgiBot X2 humanoid robot, which has 29 joints, stands 1.31~m tall, and weighs approximately 39~kg. A forward-facing Intel RealSense D435i depth camera is mounted at the waist and pitched downward by $50^\circ$ relative to the horizontal plane. The camera provides horizontal and vertical fields of view of $87^\circ$ and $58^\circ$, respectively. Raw $640\times360$ depth images are filtered, downsampled, and hole-filled using the Intel RealSense SDK, and then cropped to a resolution of $36\times36$ before being fed into the policy. Depth observations are streamed to the controller at 60~Hz.

For onboard execution, we establish a 1~kHz communication pipeline between the motion control unit (RK3588) and the inference unit (Jetson AGX Orin) via lightweight ROS~2 topic-based messaging. Policy inference runs with ONNX Runtime at 50~Hz and outputs target joint positions for all 21 policy-controlled joints. During deployment, each actuated joint is assigned an individual proportional gain $K_p$, derivative gain $K_d$, and action scale, as summarized in Table~\ref{table14}.


\section{Cross-Platform Validation}
\label{AppendixE}

\setlength{\columnsep}{6pt}
\begin{wraptable}{r}{0.36\textwidth}
\vspace{-14pt}
\begin{minipage}{0.36\textwidth}
\centering
\renewcommand{\arraystretch}{1.25}
\setlength{\tabcolsep}{1.8pt}
\footnotesize
\begin{tabular}{ccc}
\specialrule{1.2pt}{0pt}{1pt}
\textbf{Terrain} & \textbf{Settings} & \textbf{Success Rate} \\
\specialrule{0.8pt}{1pt}{0pt}
Stairs Up    & 15 / 40 cm     & 100.0\% {\tiny(20/20)} \\
Stairs Down  & 15 / 40 cm     & 100.0\% {\tiny(20/20)}\\
Gap          & 80 cm          & 100.0\% {\tiny(20/20)}\\
Platform     & 45 cm          & 95.0\% {\tiny(19/20)}\\
\specialrule{1.2pt}{0pt}{0pt}
\end{tabular}
\captionsetup{
    width=\linewidth,
    singlelinecheck=false,
    skip=4pt,
    font=footnotesize
}
\caption{Cross-platform real-world lab-level traversal performance on DEEP Robotics DR02. Success rates are computed over 20 trials per terrain.}
\label{table15}
\end{minipage}
\end{wraptable}

To evaluate the generalizability of SSR across robotic platforms, we conduct cross-platform validation on the full-size DEEP Robotics DR02 humanoid. We retrain SSR from scratch using the same optimization hyperparameters, rewards, action scales, and camera/depth preprocessing as for AgiBot X2. The same source motions are automatically retargeted to DR02; only the deployment PD gains are adapted. To account for its longer feet, the minimum stair tread in the training curriculum is increased from 25 to 33 cm, while all other terrain parameters remain unchanged. Compared with AgiBot X2, DR02 differs substantially in physical scale and dynamics, standing at roughly 1.8 m tall and weighing around 70 kg. Despite these embodiment differences, SSR can be instantiated effectively on the new platform. As reported in Table~\ref{table15} and illustrated in Fig.~\ref{fig11}, the policy achieves high success rates across representative lab-level terrains, while maintaining precise foothold placement and coordinated, natural whole-body motion. These results indicate that SSR is not tightly coupled to a specific hardware configuration and exhibits promising cross-platform applicability across distinct humanoid embodiments.

\begin{figure}[H]
  \centering
  \includegraphics[width=0.61\textwidth]{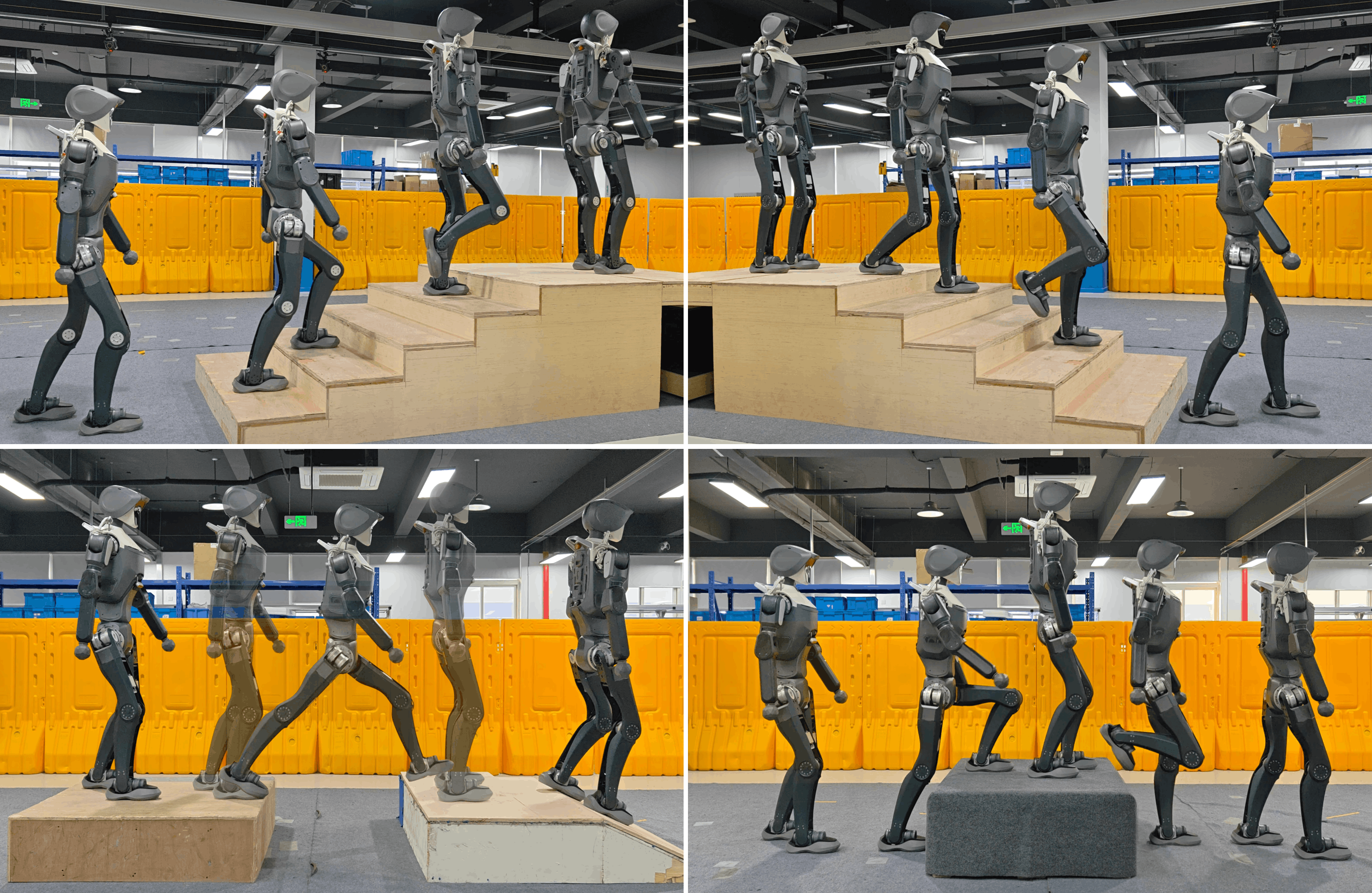}
  \captionsetup{
    width=\linewidth,
    singlelinecheck=false,
    font=footnotesize,
    skip=4pt
  }
  \caption{Key frames of cross-platform laboratory deployment on the full-size DEEP Robotics DR02 humanoid. The SSR policy successfully traverses stairs, an 80~cm gap, and a 45~cm platform, demonstrating its cross-platform generalization.}
  \label{fig11}
\end{figure}

\end{document}